\documentclass[twoside]{article}

\usepackage{amsthm}

\usepackage{amssymb,graphicx,color}
\usepackage{amsfonts}
\usepackage{latexsym}
\usepackage{xcolor}
\usepackage[colorlinks=true]{hyperref}
\definecolor{linkblue}{RGB}{0,0,128}
\hypersetup{
    linkcolor=linkblue,
    citecolor=linkblue,
    urlcolor=linkblue
}
\usepackage{mathtools}
\mathtoolsset{showonlyrefs}
\usepackage{amsmath}
\usepackage[disable]{todonotes}
\usepackage{algorithm}
\usepackage{algpseudocode}
\usepackage{graphicx}
\usepackage{placeins}
\usepackage{enumitem}
\usepackage{thmtools}
\usepackage{thm-restate}
\usepackage{booktabs,tabularx}
\usepackage{url}

\newtheorem{theorem}{Theorem}

\newtheorem{proposition}[theorem]{Proposition}

\newtheorem{definition}[theorem]{Definition}

\DeclareMathOperator{\MMD}{MMD}
\DeclareMathOperator{\supp}{supp}
\DeclareMathOperator{\erfc}{erfc}
\DeclareMathOperator{\erfi}{erfi}

\usepackage[preprint]{aistats2027}

\usepackage[round]{natbib}

\begin{document}

\twocolumn[

\aistatstitle{On the Optimisation Landscape of Overparameterised Weighted MMD}

\aistatsauthor{Tharrshann Jayan Logarajah \And Jeremias Knoblauch \And  Louis Sharrock}

\aistatsaddress{University College London \And University College London \And University College London}

]

\begin{abstract}
   We study the optimisation landscape of the maximum mean discrepancy (MMD) over weighted atomic measures. For the one-dimensional Gaussian kernel, we show that, for any compactly supported target distribution, sufficiently many particles ensure that every local minimum has arbitrarily small MMD. Moreover, any stationary point whose MMD remains bounded away from zero is unstable under an unweighted gradient flow over the particle locations. These results hold both when the weights are constrained to the simplex, and when they are unconstrained but optimised at each step. When the target is itself finitely supported, we further strengthen this result: beyond an explicit particle threshold, there are no spurious local minima, and every spurious stationary point is unstable. To our knowledge, these are the first finite-particle overparameterisation guarantees of their kind for weighted MMD optimisation.
\end{abstract}

\section{Introduction}

A common problem in statistics, machine learning, and numerical analysis is approximating a probability measure using a finite collection of representative points.
This task, commonly referred to as quantisation \citep{graf2000foundations}, arises in various applications, including experimental design \citep{kiefer1959optimum}, numerical integration and quadrature \citep{mahsereci2026bayesian}, and the compression of Monte Carlo output \citep{riabiz2022optimal}.

A natural way to formalise this problem is to choose a discrepancy between probability measures and minimise it over a suitable approximating family.
We focus here on the maximum mean discrepancy (MMD), a kernel-based integral probability metric widely used for comparing probability distributions \citep{gretton2007kernel,JMLR:v13:gretton12a}.
Given a reproducing kernel Hilbert space (RKHS) $\mathcal H_k$ with kernel $k$, the MMD between probability measures $P$ and $Q$ is defined by
\begin{equation}
    \label{eq:intro-mmd}
    \MMD(P,Q)=\sup_{\|f\|_{\mathcal H_k}\leq 1}\left|\int f\,\mathrm dP-\int f\,\mathrm dQ\right|.
\end{equation}
For a characteristic kernel, the $\MMD$ defines a metric over probability measures.
Moreover, for functions in the RKHS, a small $\MMD$ approximation error implies a low integration error when expectations under the target $P^*$ are replaced by expectations under $P$.
This, in addition to its generality, practicality, and favourable sample complexity, motivates the use of the MMD over alternative discrepancies in numerical integration and related problems \citep[e.g.,][]{chen2026stationarymmdpoints}.

At the level of probability measures, MMD minimisation has a particularly simple optimisation structure.
For a convex class $\mathcal Q\subseteq\mathcal P(\mathcal X)$, consider
\begin{equation}
    \label{eq:intro-measure-mmd}
    \min_{P\in\mathcal Q}\MMD^2(P,P^*).
\end{equation}
The objective is convex in $P$, since the kernel mean embedding is linear in the measure and the squared RKHS norm is convex.
In particular, if $\mathcal Q=\mathcal P(\mathcal X)$ and $k$ is characteristic, then the unique minimiser is simply $P=P^*$.

In practice, exactly representing an arbitrary probability measure is rarely feasible, and one instead seeks an approximate representation involving finitely many particles.
Let $\boldsymbol{x}=(x_1,\ldots,x_N)\in\mathcal X^N$ denote the locations of the particles, and $\mathbf w=(w_1,\ldots,w_N)\in\Delta^{N-1}$ their weights, where $\smash{\Delta^{N-1}=\{\mathbf w\in[0,1]^N:\sum_{i=1}^N w_i=1\}}$.
In place of \eqref{eq:intro-measure-mmd}, we are then interested in solving
\begin{equation}
    \label{eq:mmdquant}
    \min_{\mathbf w\in\Delta^{N-1},\,\boldsymbol{x}\in\mathcal X^N}\MMD^2\big(\sum_{i=1}^N w_i\delta_{x_i},P^*\big).
\end{equation}
This problem is commonly referred to as weighted quantisation  \citep[e.g.,][]{belhadji2026weighted}. 
To simplify the analysis, it is common also to consider a relaxed formulation in which the weights are unconstrained and optimised at each particle configuration \citep{belhadji2026weighted}.
The passage from \eqref{eq:intro-measure-mmd} to \eqref{eq:mmdquant} fundamentally changes the nature and the difficulty of the optimisation problem.
Indeed, while the objective in \eqref{eq:intro-measure-mmd} is convex as a function of the measure, the restriction to measures supported on at most $N$ points yields a nonconvex feasible class, and the resulting objective \eqref{eq:mmdquant} is generally nonconvex in the particle locations.

The difficulty of this problem is very well recognised.
In kernel quadrature and MMD quantisation, \citet{belhadji2019kernel}, \citet{teymur2021optimal}, and \citet{belhadji2026weighted} all emphasise the nonconvexity of optimising particle locations, even when the weights can be optimised analytically or the points are selected sequentially.
Reflecting this difficulty, \citet{chen2026stationarymmdpoints} relax the goal of global minimisation and instead analyse stationary points, proving a super-convergence property for numerical integration and a finite-particle guarantee for computing them by MMD gradient flow.
More generally, much of the existing literature either develops algorithms designed to mitigate the nonconvex optimisation problem, or seeks useful guarantees without requiring a global minimiser.
Much less well understood is the optimisation landscape itself: namely, whether increasing the number of particles can eliminate poor local minima for finite $N$.

A complementary line of work studies MMD optimisation through measure-valued dynamics and their particle approximations, often obtaining guarantees at the population level \citep{arbel2019maximummeandiscrepancygradient,gladin2024interactionforcetransportgradientflows,kang2026gradientflowperspectiveminimum}.
Such results provide insight into the behaviour of particle methods as the number of particles becomes large and, in several cases, establish convergence of finite-particle systems towards underlying measure-valued dynamics with stronger convergence guarantees.
They do not, however, directly describe the optimisation landscape at any fixed and finite number of particles $N$.
In particular, they do not rule out poor local minima of the finite-particle objective, nor provide an explicit particle threshold beyond which every local minimum approximates the target accurately.

This leads to the central question of the present work:
\begin{quote}
    \textit{Does the finite-particle optimisation landscape of MMD quantisation become benign once sufficiently many particles are used?}
\end{quote}
More specifically, how large must the number of particles be to guarantee that every local minimum is accurate?
And, beyond local minima, are stationary points with large MMD values necessarily unstable under the associated gradient flow?

\begin{table*}[!ht]
\caption{Summary of the one-dimensional landscape and instability guarantees.
Instability refers to the \emph{unweighted particle flow} \(\dot\theta_i=-\Delta'(\theta_i)\). 
Relaxed weights are unrestricted and optimised at each particle configuration; the simplex results concern the joint objective in the weights and positions. 
The rates are asymptotic as \(\varepsilon \to 0\), with \(D\) and the Gaussian bandwidth fixed.
}
\label{tab:mmd-landscape-results}
%
\begin{center}
\renewcommand{\arraystretch}{1.1}
\begin{tabular}{
    p{0.32\textwidth}
    p{0.29\textwidth}
    p{0.33\textwidth}
}
\textbf{Setting and result}
&
\textbf{Particle threshold}
&
\textbf{Guarantees} \\
\hline
\textbf{Relaxed weights}
\(K\)-atom target, ETP kernel
(Theorem~\ref{thm: relaxbound})
&
\(N\geq K\)
&
No spurious local minima.
Every spurious stationary point is unstable.
\\[1mm]
\textbf{Relaxed weights}
Target with compact support of diameter at most \(D\), Gaussian kernel
(Theorem~\ref{thm: relaxepsbound})
&
\(N>N_\varepsilon(D)
=O_D\!\left(
\frac{\log(1/\varepsilon)}
{\log\log(1/\varepsilon)}
\right)\)
&
Every local minimum has \(\MMD<\varepsilon\).
Stationary points with \(\MMD\geq\varepsilon\) are unstable.
\\[1mm]
\textbf{Simplex weights}
Target with compact support of diameter at most \(D\), Gaussian kernel
(Theorem~\ref{thm: SimpEpsBound})
&
\(N>N_\varepsilon(D)=O_D\!\left(\varepsilon^{-4}\log(1/\varepsilon)\right)\)
&
Every local minimum has \(\MMD<\varepsilon\).
Stationary points with \(\MMD\geq\varepsilon\) are unstable.
\\[1mm]
\textbf{Simplex weights}
Finite atomic target, Gaussian kernel
(Theorem~\ref{thm: SimpBound})
&
\(N>N(P^*)\)
&
No spurious local minima.
Every spurious stationary point is unstable. 
\\
\hline
\end{tabular}
\end{center}
\end{table*}

Our main results (cf. Table~\ref{tab:mmd-landscape-results}) give affirmative answers to these questions for the one-dimensional Gaussian kernel.
To be specific, our main contributions can be summarised as follows.
\begin{enumerate}
    \item For the one-dimensional Gaussian kernel, we show that for every $D>0$ and $\varepsilon>0$ there exists a finite threshold $N_\varepsilon(D)$ such that, whenever $N>N_\varepsilon(D)$ and $P^*$ has support diameter at most $D$, every local minimum of \eqref{eq:mmdquant} satisfies $\MMD<\varepsilon$.
    We also derive computable bounds for $N_{\varepsilon}(D)$. 
    We further show that every stationary point with $\MMD\geq\varepsilon$ is unstable for the \emph{unweighted particle gradient flow};
    see Section~\ref{sec:problem-setup} for a precise definition.
    These results hold both for simplex-constrained weights and the relaxed problem.
    To our knowledge, this is the first finite-particle threshold for MMD quantisation guaranteeing that every local minimum lies below a prescribed error for sufficiently large $N$.
    \item 
    In the well-specified setting, where the target admits an atomic representation $P^*=\sum_{j=1}^K w_j^*\delta_{x_j^*}$ for some finite $K$, we show that there exists a finite threshold $N(P^*)$ such that, whenever $N>N(P^*)$, there are no spurious local minima.
    That is, every local minimum exactly represents $P^*$.
    We also provide computable bounds on this threshold; in the relaxed case, the threshold is simply $N\geq K$.
    Moreover, this threshold is sharp.
    We also show that all spurious stationary points are unstable under unweighted particle gradient flow. 
    Once again, our results hold both for simplex-constrained weights and for the relaxed problem with unconstrained weights.
    \item 
    We establish structural properties of stationary points.
    In particular, we prove a maximum principle implying that every distinct particle location is \emph{active}; that is, it has non-zero total weight.
    We also show that any \emph{collision}, where $x_i=x_j$ for some $i\neq j$, is either globally optimal or, if spurious, is neither a local minimum nor stable under the unweighted particle gradient flow.
    Several of these results extend beyond the Gaussian kernel: for unconstrained weights they hold for extended totally positive (ETP) kernels, while in the simplex-constrained setting they hold additionally under translation invariance and normalisability, together with positivity of the associated Karush--Kuhn--Tucker (KKT) constant, which we verify for the Gaussian kernel.
\end{enumerate}

\section{Related work}
\label{sec:related-work}

The MMD has been widely studied since the foundational work of \citet{gretton2007kernel,JMLR:v13:gretton12a}, including for quantisation of probability measures.
\citet{chen2010supersamples} introduce kernel herding for constructing point sets with fast RKHS integration rates, while \citet{bach2012equivalence} interpret kernel herding as conditional-gradient optimisation of the corresponding convex optimisation problem over measures.
\citet{teymur2021optimal} study sequential MMD quantisation, \citet{xu2022accurate} derive finite-particle MMD quantisation bounds, and \citet{pronzato2023performance} analyses rates for kernel herding and sequential Bayesian quadrature.

These works primarily address how accurately a target can be approximated using particular finite-particle constructions.
A separate question is how to optimise the resulting finite-particle MMD objective.
Amongst recent work, \citet{belhadji2026weighted} is particularly closely related to our relaxed formulation.
They consider the case where the weights are unconstrained and optimised for each particle configuration, and introduce the \emph{mean shift interacting particles} algorithm.
This corresponds to preconditioned gradient descent on the resulting weight-optimised objective.
Related finite-particle optimisation problems are studied by \citet{belhadji2019kernel,belhadji2026discretize}.

These papers are primarily algorithmic: they do not characterise the landscape of the underlying objective itself.
Closer to this question, \citet{chen2026stationarymmdpoints} analyse stationary MMD points, proving a super-convergence property for numerical integration, together with a non-asymptotic guarantee for computing such points via MMD gradient flow.
Our focus is instead on the local minima of the finite-particle objective, together with the stability of its stationary points.
The closest work from an optimisation landscape perspective is \citet{alon2024optimizationlandscapemaximummean}, which establishes benign MMD landscapes for several structured Gaussian parametric models.
Our setting instead permits, in the well-specified case, an arbitrary finite number of target atoms with arbitrary locations and unequal weights, while our approximate guarantees require only compact support.

A complementary line of work studies MMD optimisation through measure-valued dynamics and their particle approximations.
\citet{arbel2019maximummeandiscrepancygradient} analyse the gradient flow of the MMD and the corresponding particle approximation.
More recently, several works obtain stronger guarantees by modifying the dynamics or discrepancy \citep{kang2026gradientflowperspectiveminimum,gladin2024interactionforcetransportgradientflows,chen2025deregularized,tian2026sobolev}.
Meanwhile, \citet{rosenzweig2026wassersteingradientflowsmaximum} establish convergence to collision-free critical sets for MMDs induced by a class of nonsmooth energy kernels.
In contrast, we characterise the landscape of the MMD objective directly for a fixed number of particles, rather than the convergence of a particular flow or its approximation to a population limit.

Finally, several of our proof techniques originate in the super-resolution and approximation-theory literature.
\citet{schiebinger2015superresolutionseparation} and \citet{duval2017characterizationnondegeneratesourcecondition} develop interpolation and Chebyshev-system techniques for recovery of sparse measures, including Gaussian measurement kernels.
We adapt related tools to constrain stationary configurations of the MMD.
There is also a connection to critical points of Gaussian sums: \citet{wang2026finitegaussianmixturesfiniteness} recently proved finiteness of the critical set of finite isotropic Gaussian mixtures.
Our witness functions are signed Gaussian combinations, so results for positive Gaussian mixtures do not directly apply.

\section{Problem setup}
\label{sec:problem-setup}
We now introduce the basic problem setup, as well as some relevant definitions.
For particle locations $\smash{\boldsymbol{x}=(x_1,\ldots,x_N)}$ and weights $\smash{\mathbf w=(w_1,\ldots,w_N)}$, write $P_{\mathbf w,\boldsymbol{x}}=\sum_{i=1}^N w_i\delta_{x_i}$ for the corresponding weighted atomic measure.
Define
\begin{equation}
    \label{eq:mmd-objective}
    F(\mathbf{w},\boldsymbol{x})
    =
    \MMD^2(P_{\mathbf w,\boldsymbol{x}},P^*).
\end{equation}
We can write this function in the form $F(\mathbf w,\boldsymbol{x})=\|\Delta\|_{\mathcal H_k}^2$, where 
$\Delta\in\mathcal{H}_k$ is the function on $\mathcal{X}$ defined by
\begin{equation}
    \Delta(x)=\mu_{P_{\mathbf w,\boldsymbol{x}}}(x)-\mu_{P^*}(x), \qquad x\in\mathcal{X}.
\end{equation}
This is known as the \emph{unnormalised witness function} \citep[e.g.,][]{JMLR:v13:gretton12a}.
The terminology arises from the variational characterisation of the MMD in \eqref{eq:intro-mmd}: whenever $\Delta$ is non-zero, the normalised function $\Delta/\|\Delta\|_{\mathcal H_k}$ attains the supremum, thereby witnessing the discrepancy between the two distributions.
This function will allow us to translate stationarity of the objective into constraints on $\Delta$, to which we can apply interpolation and zero-counting arguments.

Whenever the relevant derivatives exist, it is straightforward to compute
\begin{equation*}
    \frac{\partial F}{\partial w_i}=2\Delta(x_i),
    \qquad
    \frac{\partial F}{\partial x_i}=2w_i\Delta'(x_i).
\end{equation*}
Below, we will frequently refer to the \emph{unweighted particle gradient flow}, by which we mean
\begin{equation*}
    \dot{x}_i=-\Delta'(x_i),
    \qquad
    i=1,\ldots,N.
\end{equation*}
We will refer to a stationary configuration as \emph{unstable} under this flow if there exists a neighbourhood of the configuration such that, for some arbitrarily small perturbations of the particle locations, the resulting trajectory eventually leaves this neighbourhood.

We consider two main regimes: arbitrary compactly supported targets, and the setting in which the target admits an atomic representation $P^*=\sum_{j=1}^K w_j^*\delta_{x_j^*}$ for some finite $K$.
In the latter case, we say that the problem is \emph{well specified}, since for $N\geq K$ the target can be represented exactly by the approximating atomic measure, yielding zero MMD.
Note that this setting covers fitting homoscedastic Gaussian mixtures with fixed variance.
Indeed, for the Gaussian kernel $k_\sigma$, the quantity
$\smash{\MMD^2_{k_\sigma}(\textstyle\sum_{i=1}^N w_i\mathcal N(x_i,\tau^2),\sum_{j=1}^K w_j^*\mathcal N(x_j^*,\tau^2))}$ is a positive scalar multiple of
\begin{equation*}
    \MMD^2_{k_{\sqrt{\sigma^2+2\tau^2}}}(\textstyle \sum_{i=1}^N w_i\delta_{x_i},\sum_{j=1}^K w_j^*\delta_{x_j^*}).
\end{equation*}
Thus, our guarantees in the well-specified case transfer directly to the Gaussian-mixture problem.

We will work on $\mathcal{X}=\mathbb R$ and, unless otherwise specified, use the Gaussian kernel $k(x,y)=\exp(-(x-y)^2)$, corresponding to the normalisation $2\sigma^2=1$.
This is without loss of generality: for a general bandwidth, the transformation $x\mapsto x/\sqrt{2\sigma^2}$ reduces the problem to this form, with the corresponding support-diameter bounds rescaled in the same way.

\section{Relaxed variant}
\label{sec:relaxed-variant}
We first consider the \emph{relaxed} problem in which the weights are unconstrained and optimised for each fixed particle configuration.
That is, the objective is
\begin{equation}
    \label{eq:relaxed-objective}
    \bar F(\boldsymbol{x})
    =
    \min_{\mathbf w\in\mathbb R^N}
    \MMD^2\big(\textstyle\sum_{i=1}^N w_i\delta_{x_i},P^*\big).
\end{equation}
This is the weight-optimised objective underlying, for example, the \emph{mean shift interacting particles} method introduced by \citet{belhadji2026weighted}.
We begin by establishing structural consequences of stationarity, before using these to obtain a sharp result in the well-specified setting.
We then turn to arbitrary compactly supported targets.

Our analysis hinges on the observation that optimality of the weights and stationarity of the locations give
\begin{equation}
    \label{eq:relaxed-stationarity}
    \Delta(x_j)=0,
    \qquad
    w_j\Delta'(x_j)=0,
    \qquad
    j=1,\ldots,N.
\end{equation}
Thus, at every particle location with non-zero weight, both $\Delta$ and its derivative vanish, yielding a zero of multiplicity at least two.
As we now show, sufficiently many distinct such locations impose more zeros, counted with multiplicity, than a non-zero witness function can support.

\subsection{Structural properties of stationary points}
We now establish a maximum principle for the witness.
Full proofs of this and subsequent results are given in App.~\ref{sec:proofs}.
\begin{restatable}{proposition}{relaxmaxprinciple}
\label{prop: relax max principle}
    Suppose $\boldsymbol{x}$ is a stationary point of $\bar F$ with $\bar F(\boldsymbol{x})>0$, and let $\mathbf w$ attain the minimum in \eqref{eq:relaxed-objective}.
    Suppose also that $k$ is ETP.
    Then $\Delta\leq0$, every distinct particle location is active, with positive total weight, and, for all $j=1,\dots,N$,
    \begin{equation*}
        \Delta'(x_j)=0,
        \qquad
        \Delta''(x_j)<0.
    \end{equation*}
    Further, all particles lie in the convex hull of $\supp(P^*)$.
\end{restatable}
Prop.~\ref{prop: relax max principle} establishes that every particle location is a strict local maximum of $\Delta$, and that every distinct particle location has positive total weight.
The latter conclusion is related to \citet{Karvonen2019}, who establish positivity of optimally weighted quadrature rules under different assumptions.
In particular, they require the target measure to be compactly supported and to admit a Lebesgue density on a bounded interval, and the moving particles to already lie within that interval.
Here, positivity and containment in the convex hull both follow directly from stationarity, although our assumptions on the kernel are stronger.

The fact that every distinct particle location is active also clarifies the role of \emph{collisions}; that is, configurations with $x_i=x_j$ for some $i\neq j$.
Whilst colliding particles may have weights of different signs, the maximum principle ensures that their total weight at the common location is positive, so merging them cannot remove that location from the resulting measure.

We next show that collisions are \emph{unstable} under the unweighted particle gradient flow.
\begin{restatable}{proposition}{runstablecollisionpoints}
    \label{prop: runstablecollisionpoints}
    Under the assumptions of Prop. \ref{prop: relax max principle}, a collision is unstable under the unweighted particle gradient flow.
\end{restatable}
\vspace{-4mm}
\begin{proof}[Proof sketch]
Following ideas used for collision singularities in Prony systems \citep{batenkov2013geometrysingularitiespronymapping} and exponential least-squares approximation \citep{1f9723a1-a04e-3d21-9c1b-61527608e570}, a pair of colliding kernel functions can be replaced by a divided-difference basis.
This remains non-degenerate at the collision and converges to a kernel derivative.
Smoothness of an ETP kernel gives the required continuity of $\Delta''$; Prop.~\ref{prop: relax max principle} then implies that separating the particles is unstable.
\end{proof}
\vspace{-2mm}
To rule out local minima containing redundant particles, we also need to exclude zero individual weights.
\vspace{-1mm}
\begin{restatable}{proposition}{posweightrelaxed}
\label{prop: posweightrelaxed}
Suppose $k$ is real, analytic, and universal, or that $k$ is ETP and the problem is well specified.
Then, at any nonzero local minimum of $\bar F$, every weight vector $\mathbf w$ attaining the minimum in \eqref{eq:relaxed-objective} satisfies $w_i\neq0$ for all $i=1,\ldots,N$.
Moreover, the particle locations are distinct.
\end{restatable}
\vspace{-4mm}
\subsection{Well-specified targets}
We first use these structural results to obtain an exact landscape characterisation in the well-specified setting.
Proposition~\ref{prop: posweightrelaxed} ensures that increasing $N$ genuinely increases the number of zero constraints imposed on the witness function $\Delta$.
In particular, a putative local minimum cannot simply absorb additional particles by assigning them zero weight.
We can therefore combine the stationarity conditions with the zero-counting property of ETP kernels.
\begin{theorem}
\label{thm: relaxbound}
Suppose the kernel is ETP and the problem is well specified with $P^*$ supported on $K$ points.
If $N\geq K$, then $\bar F$ has no spurious local minima, and every spurious stationary point is unstable under the unweighted particle gradient flow.
\end{theorem}
\vspace{-4mm}
\begin{proof}
First consider a nonzero local minimum.
By Prop.~\ref{prop: posweightrelaxed}, all particle locations are distinct and every optimal weight is nonzero.
Hence stationarity forces a double zero of $\Delta$ at each of the $N$ particle locations.
Since $\Delta$ is a nonzero combination of at most $N+K$ distinct kernel sections, extended total positivity implies that it has at most $N+K-1$ zeros counting multiplicity.
Thus $2N\leq N+K-1$, which requires $N\leq K-1$.
Therefore, when $N\geq K$, every local minimum must satisfy $\bar F=0$.

Now consider a spurious stationary point.
If all particle locations are distinct, Prop.~\ref{prop: relax max principle} implies that every location is active, so the same zero-counting argument again gives a contradiction when $N\geq K$.
Hence every spurious stationary point must contain a collision, which is unstable by Prop.~\ref{prop: runstablecollisionpoints}.
\end{proof}
\vspace{-2mm}
Theorem~\ref{thm: relaxbound} shows that, in the relaxed well-specified problem, using enough particles to represent the target suffices to eliminate all spurious local minima.
For completeness, we provide a numerical illustration of this result, as well as our other main optimisation landscape results (cf. Theorems~\ref{thm: relaxepsbound}--\ref{thm: SimpBound}), in App.~\ref{app:numerics}.
The observed behaviour is consistent with our theory, recovering in particular the sharp \(N\geq K\) threshold established in Theorem~\ref{thm: relaxbound}.

\subsection{Arbitrary compactly supported targets}
We now turn to arbitrary compactly supported targets.
In this case, the same zero-counting principle yields a quantitative approximation guarantee without requiring the well-specified setting.

\begin{restatable}{theorem}{relaxepsbound}
\label{thm: relaxepsbound}
Let $D>0$ and $\varepsilon>0$.
There exists a finite threshold $N_\varepsilon(D)$ such that, uniformly over all target distributions with compact support of diameter at most $D$, whenever $N>N_\varepsilon(D)$ every local minimum of $\bar F$ satisfies
\begin{equation*}
    \MMD(P,P^*)<\varepsilon.
\end{equation*}
Moreover, every stationary point with $\MMD(P,P^*)\geq\varepsilon$ is unstable under the unweighted particle gradient flow.
\end{restatable}
\vspace{-4mm}
\begin{proof}[Proof sketch]
This result is also proved via a zero counting argument.
First, the maximum principle (cf. Prop.~\ref{prop: relax max principle}) confines the active particles to the convex hull of the target support.
Meanwhile, compactness and analyticity give an upper bound on the number of distinct active locations that a stationary witness with $\MMD(P,P^*)\geq\varepsilon$ can support, depending only on $D$ and $\varepsilon$.
At a nonzero local minimum, all $N$ particle locations must be distinct (cf. Prop.~\ref{prop: posweightrelaxed}), yielding a contradiction once $N$ exceeds this bound.
Meanwhile, for a general stationary point, exceeding the bound forces a collision, which must be unstable (cf. Prop.~\ref{prop: runstablecollisionpoints}).
\end{proof}
\vspace{-2mm}

\section{Simplex variant}
\label{sec:simplex-variant}
We now return to the original problem, where $\mathbf w\in\Delta^{N-1}$.
The simplex constraint makes the analysis substantially more difficult: weight stationarity no longer forces $\Delta(x_i)=0$, but instead introduces a common KKT level $c$.

Our proof proceeds in two stages.
We first establish structural properties common to both results.
We next use an interpolation argument to show that the simplex KKT level $c$ is positive, after which a maximum principle forces every active particle location to be a strict local maximum of the witness function and provides control of collisions.
We then relate the MMD, the KKT level $c$, and the number of distinct active particle locations.
For arbitrary compactly supported targets, this yields the approximate guarantee.
In the well-specified setting, we additionally show that every spurious stationary point must miss some target atom by a fixed neighbourhood, which provides the positive MMD lower bound needed to obtain the exact result.

Our two main results can be stated as follows.

\begin{restatable}{theorem}{MMDEpsSimplexBound}
\label{thm: SimpEpsBound}
Let $D>0$ and $\varepsilon>0$.
There exists a finite threshold $N_\varepsilon(D)$ such that, uniformly over all target distributions with compact support of diameter at most $D$, whenever $N>N_\varepsilon(D)$, every local minimum of $F$ over $\mathbf w\in\Delta^{N-1}$ and $\boldsymbol{x}\in\mathbb R^N$ satisfies
\begin{equation*}
    \MMD(P,P^*)<\varepsilon.
\end{equation*}
Moreover, every stationary point with $\MMD(P,P^*)\geq\varepsilon$ is unstable under the unweighted particle gradient flow.
\end{restatable}
\vspace{-1mm}
\begin{restatable}{theorem}{MMDSimplexBound}
\label{thm: SimpBound}
Suppose the problem is well specified.
There exists a finite threshold $N(P^*)$ such that, whenever $N>N(P^*)$, $F$ has no spurious local minima over $\mathbf w\in\Delta^{N-1}$ and $\boldsymbol{x}\in\mathbb R^N$.
Moreover, every spurious stationary point is unstable under the unweighted particle gradient flow.
\end{restatable}
\vspace{-1mm}

\subsection{Structural properties of stationary points}

We first develop the structural ingredients common to both results.
To begin, observe that the KKT conditions for the simplex weights give a constant $c$ such that, for every $j=1,\dots,N$,
\begin{align}
    \Delta(x_j)\geq c, \quad 
    w_j(\Delta(x_j)-c)=0, \quad 
    w_j\Delta'(x_j)=0. \label{Stationary0simp}
\end{align}
As in the relaxed problem, we consider the reduced measure defined by removing zero-weight particles and merging particles lying at the same coordinate.
Let $\vartheta_1,\ldots,\vartheta_A$ denote the distinct active locations and $\varsigma_1,\ldots,\varsigma_A$ their positive total weights.
For these, the stationarity conditions give, for every $j=1,\dots,A$,
\begin{align}
    \Delta(\vartheta_j)=c, \quad 
    \Delta'(\vartheta_j)=0, \label{Stationary0redsimp}
\end{align}
Unlike in the relaxed problem, we must now control the shifted witness function $\Delta-c$.
Our first ingredient is an interpolation certificate built from Gaussian kernels and their derivatives.
Given distinct centres $\bar{x}_1,\ldots,\bar{x}_K$ and distinct target locations $x_1^*,\ldots,x_K^*$, we construct a function $H$ centred at the $\bar{x}_i$ such that $H(x_i^*)=1$, $H'(x_i^*)=0$, and $H(x)<1$ away from the target locations.
This extends related super-resolution certificates \citep{schiebinger2015superresolutionseparation,duval2017characterizationnondegeneratesourcecondition} by allowing the kernel centres $\bar{x}_i$ to differ from the interpolation locations $x_i^*$.

\begin{restatable}{proposition}{Hinterp}
\label{prop: Hinterp}
Let $\bar{x}_1,\ldots,\bar{x}_K$ and $x_1^*,\ldots,x_K^*$ each consist of distinct real numbers.
Define 
\begin{equation}
    H(\cdot)
=
\sum_{i=1}^K
\left[
\alpha_i k(\cdot,\bar{x}_i)
+
\beta_i\partial_2k(\cdot,\bar{x}_i)
\right].
\end{equation}
Then there exist $\boldsymbol\alpha,\boldsymbol\beta$ such that $H(x)<1$ for all $x\notin\{x_1^*,\ldots,x_K^*\}$, and 
\begin{align*}
    H(x_i^*)=1, \qquad H'(x_i^*)=0 \qquad i=1,\ldots,K.
\end{align*}
\end{restatable}
\vspace{-3mm}
The certificate immediately gives us control of the KKT level.
Indeed, suppose we apply Prop.~\ref{prop: Hinterp} with the interpolation centres and interpolation locations equal to the distinct active particle locations.
The resulting function equals one at every active location and is strictly less than one elsewhere.
At a non-optimal stationary point, this allows us to show $\int H\,\mathrm d(P-P^*)>0$.
On the other hand, the stationarity conditions express the same quantity as $c\sum_i\alpha_i$, with $\sum_i\alpha_i=\|H\|_{\mathcal H_k}^2>0$.
It follows that the KKT level must be strictly positive.
\begin{restatable}{proposition}{cpos}
    \label{prop: cpos}
    Suppose we are at a stationary point of $F$ with $F(\mathbf{w}, \boldsymbol{x}) > 0$.
    Then $0<c<1$.
\end{restatable}
\vspace{-2.5mm}
We can now establish the simplex analogue of Prop.~\ref{prop: relax max principle}.
Unlike the relaxed case, the stationarity conditions involve the shifted witness function $\Delta-c$.
Translation invariance and normalisability allow the constant $c$ to be represented as a kernel integral against Lebesgue measure, while the condition $c\geq0$ ensures that the resulting signed measure admits the same variation-diminishing argument for ETP kernels.
\begin{restatable}{proposition}{simpmaxprinciple}
\label{prop: simp max principle}
Suppose $(\mathbf w,\boldsymbol{x})$ is a stationary point of $F$ with $F(\mathbf w,\boldsymbol{x})>0$, with weights constrained to the simplex.
Suppose also that $k$ is ETP, translation invariant, and normalisable, and that $c\geq0$.
Then $\Delta(x)\leq c$ for all $x$, every distinct particle location is active, and for every $i=1,\dots,N$,
\begin{equation*}
    \Delta(x_i)=c,
    \qquad
    \Delta'(x_i)=0,
    \qquad
    \Delta''(x_i)<0.
\end{equation*}
\end{restatable}
\vspace{-3mm}
Thus every particle location lies at a strict local maximum of the witness function at the common level $c$.
As in the relaxed problem, cf. Prop.~\ref{prop: runstablecollisionpoints}, this immediately gives strong control of collisions.
\begin{proposition}
\label{prop: unstablecollisionpoints}
Under the assumptions of Prop.~\ref{prop: simp max principle}, every collision is unstable under the unweighted particle gradient flow.
\end{proposition}
\vspace{-4mm}
\begin{proof}
Separate two colliding particles by a small distance $r=x_j-x_i$.
Under the unweighted particle flow,
$
    \dot r
    =
    -\bigl(\Delta'(x_j)-\Delta'(x_i)\bigr)
    =
    -\Delta''(\vartheta)r
$
for some $\vartheta\in(x_i,x_j)$ by the mean value theorem.
Since $\Delta''<0$ in a neighbourhood of the collision, $\dot r$ has the same sign as $r$.
Thus the collision is unstable.
\end{proof}
\vspace{-2mm}
Prop.~\ref{prop: unstablecollisionpoints} rules out collisions as stable stationary configurations under the particle flow.
To obtain the corresponding conclusion for the optimisation landscape, we must also show that collisions cannot themselves be local minima.

\begin{proposition}
\label{prop: collisionpoints}
Under the assumptions of Prop.~\ref{prop: simp max principle}, a collision is not a local minimum.
\end{proposition}
\vspace{-4mm}
\begin{proof}
    Suppose $x_i = x_j$, and $w_i, w_j > 0$.
    Perturb the locations to
    $x_i - \frac{w_j}{w_i + w_j}\varepsilon$ and $x_j + \frac{w_i}{w_i + w_j}\varepsilon$.
    If $F(\varepsilon)$ denotes the perturbed objective, then
    \begin{equation*}
        F(\varepsilon)-F(0)
        =
        \frac{w_iw_j}{w_i+w_j}\Delta''(x_j)\varepsilon^2+o(\varepsilon^2)<0,
    \end{equation*}
    for all sufficiently small $\varepsilon$.
    Thus, the perturbation decreases the objective for all sufficiently small $\varepsilon$.
    If one of the weights is zero, transfer an arbitrarily small amount of mass to that particle without changing the measure, then apply the same perturbation.
\end{proof}
\vspace{-2mm}
Together, the maximum principle and the collision results allow us to work with active, distinct particle locations when studying nonzero local minima.
This is crucial: every such location contributes the full pair of equality constraints $\Delta(x_i)=c$ and $\Delta'(x_i)=0$, allowing the number of distinct active locations to be controlled by zero counting.

\subsection{Arbitrary compactly supported targets}
We first use these results to establish a guarantee for arbitrary compactly supported targets.
The key step is to relate a lower bound on the MMD to a lower bound on the KKT level $c$, and then to an upper bound on the number of distinct active particle locations.
The first relation follows by splitting the Fourier representation of the MMD into low and high frequencies.

\begin{restatable}{proposition}{MMDFourierSplit}
\label{prop: MMDFourierSplit}
Let $P,P^*$ be probability measures on $\mathbb R^d$, let $\rho=P-P^*$, and suppose $k$ is continuous, translation invariant, and has positive spectral density.
Then, writing $\Delta=\mu_P-\mu_{P^*}$, for every $\nu>0$ it holds that
\begin{align*}
    \MMD^2(P^*,P)
    \leq&
    \frac{\|\Delta\|_1^2}{(2\pi)^d}
    \int_{|\omega|\leq\nu}\frac{1}{\hat k(\omega)}\,\mathrm d\omega\\
    &+
    \frac{\|\rho\|_{\mathrm{TV}}^2}{(2\pi)^d}
    \int_{|\omega|>\nu}\hat k(\omega)\,\mathrm d\omega.
\end{align*}
In particular, for the one-dimensional Gaussian kernel with $2\sigma^2=1$,
\begin{equation*}
    \MMD^2(P,P^*)
    \leq
    \tfrac{\|\Delta\|_1^2}{\pi}\erfi\left(\tfrac{\nu}{2}\right)
    +
    4\erfc\left(\tfrac{\nu}{2}\right).
\end{equation*}
where $\erfi(z)=\frac{2}{\sqrt{\pi}}\int_0^z e^{t^2}\,\mathrm dt$, $\erfc(z)=\frac{2}{\sqrt{\pi}}\int_z^\infty e^{-t^2}\,\mathrm dt$ denote the imaginary error function and complementary error function, respectively.
\end{restatable}
\vspace{-3mm}
The maximum principle then lets us control $\|\Delta\|_1$ through $c$ and the target support diameter, without dependence on the particular particle configuration.

\begin{restatable}{proposition}{Lcbound}
\label{prop: Lcbound}
Suppose $P^*$ has compact support of diameter at most $D$, and $(\mathbf w,\boldsymbol{x})$ is a stationary point with $F(\mathbf w,\boldsymbol{x})>0$.
Then
\begin{align*}
    \|\Delta\|_1
    \leq&
    2(\sqrt{2}+D)c
    +16c\sqrt{-\log c}\\
    &+
    6\sqrt{\pi}\erfc(\sqrt{-\log c})
    +
    4\sqrt{\pi}c\sqrt{1-2\log c}.
\end{align*}
\end{restatable}
\begin{proof}[Proof sketch]
Split the real line into regions near the target support, intermediate regions, and Gaussian tails.
The maximum principle controls the witness directly on the central region; the tails are bounded using Gaussian decay.
The main difficulty is to remove any dependence on how far the extreme particles lie from the target.
We do this using the stationarity equations at the extreme particles to obtain probability-mass bounds, after which suitable choices of the splitting points yield the stated uniform estimate.
\end{proof}
\vspace{-2mm}
The bound in Prop.~\ref{prop: Lcbound} is $O(c\sqrt{-\log c})$, and thus tends to zero as $c\rightarrow 0^{+}$.
Note, in particular, that $\erfc(\sqrt{-\log c})=O(c/\sqrt{-\log c})$, so this term is of lower order.
Together with Prop.~\ref{prop: MMDFourierSplit}, this implies that a stationary point with MMD bounded away from zero must have its KKT level bounded away from zero.
We now convert this into a bound on the number of distinct particle positions.
\begin{restatable}{proposition}{NUpperBound}
\label{prop: NUpperBound}
A positive lower bound on the KKT level $c$ gives an upper bound on the number $A$ of distinct active particle locations.
One such bound is
\begin{align*}
    A
    \leq&
    \frac{2-c}{c}\log\left(\frac{2(c+2)}{c}\right)
    +
    \frac{e}{2}
    \left(\frac{8-5c}{c}\right)^2
    \log\left(\frac{2}{c}\right).
\end{align*}
\end{restatable}
\begin{proof}[Proof sketch]
The maximum principle gives a lower bound on the probability mass in a neighbourhood of each particle.
We group overlapping neighbourhoods into connected components and construct a maximal separated packing in each component.
This controls the component diameters and hence, by Jensen-type zero counting for the shifted witness function $\Delta-c$, the number of double zeros contributed by the particles.
Summing over components gives the stated bound.
\end{proof}
For a stationary point for which $\MMD(P,P^*)\geq\varepsilon$, Props.~\ref{prop: MMDFourierSplit}--\ref{prop: Lcbound} give a positive lower bound on $c$, depending only on $D$ and $\varepsilon$.
Prop.~\ref{prop: NUpperBound} then gives an upper bound on the number of distinct active particle locations.
If the total number of particles $N$ exceeds this upper bound, at least two particles must share a location.
Prop.~\ref{prop: unstablecollisionpoints}--\ref{prop: collisionpoints} then rule out a poor local minimum or stable stationary point.
This proves Theorem~\ref{thm: SimpEpsBound}.

\subsection{Well-specified targets}

We now turn to the well-specified setting.
In this case, it remains to obtain a target-dependent positive lower bound on the MMD at every spurious stationary point.
Our first result shows that such a point must remain uniformly separated from at least one target atom.

\begin{restatable}{proposition}{neighbourhoodmiss}
\label{prop: neighbourhoodmiss}
Suppose that the problem is well specified, with $P^*=\sum_{j=1}^K w_j^*\delta_{x_j^*}$.
Then there exists $r(P^*)>0$ such that every spurious stationary point satisfies
\begin{equation}
    \max_j\min_i|x_i-x_j^*|
    \geq
    r(P^*).
\end{equation}
\end{restatable}
It may not be surprising that any spurious stationary point must miss some target atoms by some amount.
The important point here is that the lower bound $r(P^*)$ is independent of the number of particles $N$.
Thus, adding more particles cannot produce a sequence of spurious stationary points which approaches every target atom arbitrarily closely.

We next convert this geometric separation into a positive lower bound on the MMD.

\begin{restatable}{proposition}{MMDHoleBound}
\label{prop: MMDholebound}
Suppose a target atom is missed by a neighbourhood of radius $r$ as in Prop.~\ref{prop: neighbourhoodmiss}.
Then
\begin{equation*}
    \MMD(P,P^*)
    \geq
    \frac{[w^*-(ar)^{-2m}]_+}{C(a,m)},
\end{equation*}
where $[u]_+=\max\{u,0\}$, $w^*=\min_jw_j^*$, and one may take 
\begin{equation}
    C(a,m)=\frac{\sqrt{\pi}}{a}e^{\frac12m^2a^2}, \qquad m\in\mathbb N_{\geq1}, \qquad a>0.
\end{equation}
Thus, choosing $a,m$ such that $(ar)^{-2m}<w^*$ gives a strictly positive lower bound on the MMD.
\end{restatable}

Prop.~\ref{prop: MMDholebound} therefore gives a positive lower bound on the MMD at every spurious stationary point, depending only on $P^*$.
The preceding quantitative argument can now be reused: Props.~\ref{prop: MMDFourierSplit}--\ref{prop: Lcbound} convert this into a positive lower bound on $c$, while Prop.~\ref{prop: NUpperBound} gives a finite upper bound on the number of distinct active particle locations.
Beyond this bound, a collision is unavoidable, and Prop.~\ref{prop: unstablecollisionpoints}--\ref{prop: collisionpoints} rule out a spurious local minimum or stable stationary point.
This proves Theorem~\ref{thm: SimpBound}.

\section{Conclusion}
\textbf{Summary.}
We have shown that overparameterisation can fundamentally simplify the finite-particle MMD landscape.
For the one-dimensional Gaussian kernel and any compactly supported target of diameter at most $D$, there is a computable threshold $N_\varepsilon(D)$ beyond which every local minimum has MMD below $\varepsilon$.
Meanwhile, every stationary point with MMD at least $\varepsilon$ is unstable under the unweighted particle gradient flow.
These conclusions hold for both simplex-constrained and relaxed, weight-optimised formulations.
In the well-specified setting, beyond $N=N(P^*)$, all local minima are globally optimal and all spurious stationary points are unstable.
For the relaxed problem, this extends to ETP kernels with the sharp threshold $N\geq K$.
Our maximum principles and collision results explain this phenomenon: sufficiently many active, distinct particles impose more value and derivative constraints on the witness function than a poor stationary configuration can support.
Combined with algorithm-specific convergence theory, our results provide approximation guarantees for methods converging to local minima or to stationary points stable under the corresponding particle flow.

\textbf{Limitations and future work.}
The principal limitation of this work is the restriction of our strongest results to one dimension.
Our proofs rely on one-dimensional zero counting and an ordering argument for particle signs, neither of which transfers directly to higher dimensions.
Developing multidimensional analogues is thus an important open problem.
This notwithstanding, it is worth emphasising that our results provide the first finite-particle overparameterisation guarantees of their kind for weighted MMD optimisation. A second direction is to extend our results to broader classes of kernels.
The maximum principles already hold beyond Gaussian kernels, while the simplex analysis additionally requires suitable interpolation and zero-counting arguments.
In particular, Prop.~\ref{prop: Hinterp} relies on a Chebyshev-system property of Gaussian kernels, their derivatives, and the constant function.
%
%
The different dependence on $\varepsilon$ in the simplex and relaxed analyses suggests that the simplex bounds are unlikely to be tight, while in the well-specified simplex case sharper control of the interpolation matrices may improve the threshold.
By contrast, the relaxed well-specified threshold $N\geq K$ is sharp.
Finally, our results concern stationary configurations with finite particle locations.
For asymptotically stationary trajectories, particles may escape to infinity, evading the finite-position argument.
The appendix provides a preliminary analysis; controlling such escape for specific flows would turn our results into stronger global convergence guarantees.

\clearpage
\subsection*{AI use statement}

In this work, we used generative AI tools for assistance in the writing of proofs, for implementing methods and for designing or providing feedback on research methodology or experiments.
We also used generative AI tools for brainstorming, sourcing \& searching for information, identifying relevant literature, formatting references, suggesting experimental parameters and suggesting structure.
We have not used generative AI tools for helping to develop theoretical models or conceptual frameworks, formulating mathematical claims, providing critical ingredients for proving mathematical claims, proposing or refining hypotheses, generating synthetic data sets, assisting with translation, cleaning and reformatting datasets, supporting qualitative and thematic data analysis, or interpreting results.
All LLM assisted work was checked carefully, including any routine or standard results or applications it assisted with.
We take responsibility for the final content of this work, including text, claims, or artifacts produced with the aid of generative AI.

\bibliographystyle{plainnat}
\bibliography{references-mmd}

\clearpage
\appendix
\thispagestyle{empty}

\onecolumn

\aistatstitle{Supplementary Materials}

\section{Background and preliminaries}
\subsection{Kernel basics}
In this appendix we collect some basic definitions relating to kernels, as well as some other background material necessary for the subsequent proofs.
For a more detailed exposition, we refer the interested reader to the included references.
The definitions below are taken from \citet{schrab2025practicalintroductionkerneldiscrepancies,Gretton2024SampleModel,Gretton2025TwoSamples}.

Let $\mathcal{X}$ be a non-empty set.
\begin{definition}[Hilbert space]
    A vector space $\mathcal{H}$ is called a Hilbert space if it has an inner product $\langle \cdot, \cdot \rangle_\mathcal{H}$ and is complete.
\end{definition}

\begin{definition}[Reproducing Kernel Hilbert Space]
    For a Hilbert space $\mathcal{H}$ of real-valued functions on $\mathcal{X}$, a function $k: \mathcal{X} \times \mathcal{X} \to \mathbb{R}$ is a reproducing kernel of $\mathcal{H}$ if
    \begin{enumerate}
        \item $k(x, \cdot) \in \mathcal{H}$ for all $x \in \mathcal{X}$;
        \item $\langle f, k(x, \cdot) \rangle_\mathcal{H} = f(x)$ for all $f \in \mathcal{H}$ and all $x \in \mathcal{X}$.
    \end{enumerate}
    In this case, $\mathcal{H}$ is called a Reproducing Kernel Hilbert Space (RKHS).
\end{definition}

\begin{definition}[Mean Embedding]
Let $k$ be a kernel, and $P$ be a probability distribution.
Then the (kernel) \emph{mean embedding} is defined by
    \begin{equation*}
        \mu_P(t) = \mathbb{E}_{x \sim P} k(x, t).
    \end{equation*}
\end{definition}
The existence of such a mean embedding follows by the Riesz Representation Theorem under reasonable conditions.
\begin{definition}[Maximum Mean Discrepancy]
Let $k$ be a kernel, and $P, Q$ be probability distributions.
Then the \emph{maximum mean discrepancy} (MMD) is defined by
    \begin{equation*}
        \MMD^2(P,Q) = \vert \vert \mu_P  - \mu_Q \vert \vert_\mathcal{H}^2.
    \end{equation*}
\end{definition}
\begin{proposition}[Reproducing property for the derivative]
    For a twice continuously differentiable kernel $k(x, x')$ with associated RKHS $\mathcal{F}$ and feature map $\varphi(x)$, and for $f \in \mathcal{F}$, we have
    \begin{equation*}
        \frac{d}{dx}f(x) = \langle f, \frac{d}{dx}\varphi(x)\rangle_\mathcal{F}, \qquad \langle \frac{d}{dx}\varphi(x), \varphi(x') \rangle_{\mathcal{F}} = \frac{d}{dx}k(x, x').
    \end{equation*}
\end{proposition}

Throughout, we will typically focus on the Gaussian kernel.
\begin{definition}[Gaussian Kernel]
The Gaussian kernel with bandwidth $\sigma$ is, for $x, y$ real vectors in some common space,
\begin{equation*}
    k_\sigma(x, y) = \exp\biggl(-\frac{\vert \vert x - y \vert \vert^2}{2\sigma^2}\biggr).
\end{equation*}
\end{definition}

Finally, we introduce the notion of a normalisable kernel.
\begin{definition}[Normalisable kernel]
    A translation-invariant kernel $k$ is called normalisable if $0 < \mathcal{Z} := \int_{\mathbb{R}} k(x, y) dx < \infty$.
\end{definition}

\subsection{Total positivity, variation diminishing, Chebyshev systems}
We now introduce further background which will be necessary for the arguments that follow.
These definitions are recalled from \citet{karlin1961moment}.

\begin{definition}[Totally Positive (TP)]
    A function $f(x,y)$ of two real variables, $x$ ranging over $X$ and $y$ ranging over $Y$, is said to be totally positive of order $k$ ($TP_k$) if for all $x_1 < \ldots < x_m, y_1 < \ldots < y_m (x_i \in X, y_j \in Y)$, and all $1 \le m \le k$,
    \begin{equation}
        f \begin{pmatrix} x_1, \ldots, x_m\\y_1, \ldots, y_m \end{pmatrix} = \begin{vmatrix} f(x_1, y_1) & \ldots &f(x_1, y_m)\\ \vdots && \vdots\\ f(x_m, y_1) & \ldots & f(x_m, y_m) \end{vmatrix} \ge 0.
    \end{equation}
    Additionally, if all indicated derivatives exist, and $X, Y$ represent open intervals on the line, then for $1 \le n \le k$,
    \begin{equation}
        \label{TPnonstrict}
        \begin{vmatrix} f(x, y) & \ldots &\frac{\partial^{n-1}}{\partial y^{n-1}}f(x, y)\\ \vdots && \vdots\\ \frac{\partial^{n-1}}{\partial x^{n-1}}f(x, y) & \ldots & \frac{\partial^{2n-2}}{\partial x^{n-1} \partial y^{n-1}}f(x, y) \end{vmatrix} \ge 0.
    \end{equation}
\end{definition}

The Gaussian kernel is $TP_\infty$.
One of the key properties of these functions we will utilise, particularly in our proofs of Props.~\ref{prop: relax max principle} and \ref{prop: simp max principle}, is the variation diminishing property.
We here present a particularly strong form that holds with additional assumptions, lifted directly from \citet{karlin1961moment}.

\begin{definition}[Variation diminishing property]
\label{def: Variation deimishing}
    Let $V(g)$ denote the number of variations of sign of a function $g(x)$ as $x$ traverses the real line from left to right.
    Let $h(x)$ be given by the absolutely convergent integral
    \begin{equation}
        \label{convergenttpintegral}
        h(x) = \int f(x, w)g(w)d\sigma(w),
    \end{equation}
    where $f(x, w)$ is $TP_k$, $\sigma(w)$ is a non-negative $\sigma$-finite measure and $V(g) \le k-1$.
    Suppose also that $h(x)$ is $k$ times continuously differentiable, the interchange of integral and derivative is permissible in \eqref{convergenttpintegral}, and \eqref{TPnonstrict} holds strictly.
    Suppose $h \not\equiv 0$.
    Let $Z(h)$ denote the number of zeros of $h(x)$, counting multiplicities.
    Then $Z(h) \le V(g)$.

\end{definition}

The following definitions are recalled from \citet{doi:10.1137/S1064827596310779}, which itself takes these from \citet{KarlinStudden1966,Krein1959}.

\begin{definition}[Chebyshev system]
A sequence of functions $\phi_1, \ldots, \phi_n$ is referred to as a Chebyshev system on the interval $[a, b]$ if each of them is continuous and the determinant
\begin{equation*}
    \begin{vmatrix} \phi_1(x_1) & \ldots & \phi_1(x_n)\\ \vdots & & \vdots\\ \phi_n(x_1) & \ldots & \phi_n(x_n) \end{vmatrix}
\end{equation*}
is nonzero for any sequence of points $a \le x_1 < \ldots < x_n \le b$.
Equivalently, any nontrivial linear combination of the functions has at most $n-1$ distinct zeros.
\end{definition}

We now define a special case that is particularly relevant when using the Gaussian kernel, as will be seen shortly.

\begin{definition}[Extended Chebyshev system]
    Given a set of functions $\phi_1, \ldots, \phi_n$ which are continuously differentiable as needed on an interval $[a, b]$, and given a sequence of ordered points $a \le x_1 \le  \ldots \le x_n \le b$, define the sequence $m_1, \ldots, m_n$ by
    \begin{equation}
        \label{msequenceforet}
        \begin{aligned} m_1 = 0, \qquad m_j = \begin{cases} 0 & x_j \neq x_{j-1}\\ m_{j-1} + 1 & x_j = x_{j-1} \end{cases}.\\ \end{aligned}
    \end{equation}
    Let the matrix $C(x_1, \ldots, x_n) = [c_{ij}]$ be defined by the formula
    \begin{equation*}
        c_{ij} = \frac{d^{m_j}\phi_i}{dx^{m_j}}(x_j).
    \end{equation*}
    Then $\phi_1, \ldots, \phi_n$ is an extended Chebyshev system on $[a, b]$ if $\det(C(x_1, \ldots, x_n))$ is nonzero for all such sequences $x_i$.
\end{definition}

As remarked in \citet{doi:10.1137/S1064827596310779}, this is essentially just a Chebyshev system where, if a point is repeated, the repeated column is replaced with the derivative of the function, successively taking the next derivative if it is repeated again.
This matrix is then mandated to be invertible.

Finally, we introduce the kernel property underlying our zero-counting arguments.
\begin{definition}[Extended Totally Positive (ETP)]
    Given a function $K: [a,b] \times [c,d] \to \mathbb{R}$ which is continuously differentiable as needed, and given a sequence of points $x_1, \ldots, x_n$ such that $a \le x_1 \le\ldots \le  x_n \le b$, let the sequence $m_1, \ldots, m_n$ be defined in \eqref{msequenceforet}.
    Let the functions $\phi_1, \ldots, \phi_n$ be defined by the formula
    $
    \phi_j(t) = \frac{\partial^{m_j}K}{\partial x^{m_j}}(x_j, t).
    $
    Then $K$ is ETP if $\phi_1, \ldots, \phi_n$ form an extended Chebyshev system on $[c, d]$ for all such sequences of $x_i$.
\end{definition}

Crucially, the Gaussian kernel is ETP  \citep[e.g.,][]{doi:10.1137/S1064827596310779}.
If we have a reproducing kernel on the real line, this is a stronger condition than TP$_\infty$ and the strictness asked for in the variation diminishing property.
Indeed, positive semidefiniteness of the kernel makes the relevant determinants non-negative when $x = y$, and ETP makes them nonzero.
They are thus all positive when $x = y$, and since the real line is connected and the determinants are always nonzero, must be positive for all $x, y$.

\subsection{Zero counting using complex analysis and specialisation to Gaussian-type functions}
We will employ a bounding strategy similar to that used in \citet{gröchenig2018sharpresultssamplingderivatives}, Lemmas 4.1 and 4.3.
In particular, we make use of the following result.
\begin{proposition}
    \label{prop: jenzerocount}
    Let $R > 0$.
    Let $\delta  > 1$.
    Let $\Omega$ be an open set that contains the closure of a disc $D_{\delta R}$ and suppose that $f$ is holomorphic in $\Omega$ with $f(0) \neq 0$.
    Let $\mathbf{n}_f(r)$ denote the number of zeros (counting multiplicity) of $f$ in the disc of radius $r$.
    Assume that $\vert f(x + iy) \vert \le A_fe^{B_fy^2}$.
    Then
    \begin{equation*}
        \mathbf{n}_{f} (R) \le \frac{\log\bigl(\frac{A_f}{\vert f(0) \vert}\bigr) + \frac{1}{2}B_f\delta^2R^2}{\log(\delta)}.
    \end{equation*}
\end{proposition}
\begin{proof}
    By Jensen's formula \citep[e.g., pg. 137][]{stein2010complex}, we have that
    \begin{align*}
        {\frac{1}{2\pi} \int_{0}^{2\pi} \log \vert f(\delta Re^{i\theta}) \vert d\theta - \log \vert f(0) \vert} &= \int_{0}^{\delta R} \frac{{\mathbf{n}_{f}}(r)}{r}dr \ge \int_{R}^{\delta R} \frac{{\mathbf{n}_{f}}(r)}{r} dr \ge  {\mathbf{n}_{f}} (R) \log(\delta).
    \end{align*}
    By the assumed bound, we have $\vert f(\delta Re^{i\theta})\vert \le A_fe^{B_f\delta^2 R^2\sin^2(\theta)}$.
    Thus, integrating and substituting into the previous bound, it follows that
    \begin{align*}
        \mathbf{n}_{f} (R) &\le \frac{\log\bigl(\frac{A_f}{\vert f(0) \vert}\bigr) + \frac{1}{2}B_f\delta^2R^2}{\log(\delta)}. \qedhere
    \end{align*}
\end{proof}

\section{Proofs}
\label{sec:proofs}

\subsection{Relaxed Variant}

\relaxmaxprinciple*
\begin{proof}
     We begin by defining the reduced measure of $P$ by merging particles at the same location, and removing locations with zero total weight.
    Write $\vartheta_1,\ldots,\vartheta_A$ for its distinct locations and $\varsigma_1,\ldots,\varsigma_A$ for their nonzero total weights.
    By the stationarity conditions, we have
    \begin{equation*}
        \Delta(\vartheta_j)=0, \qquad \Delta'(\vartheta_j)=0, \qquad j=1,\ldots,A.
    \end{equation*}
    To apply the variation diminishing property, write
    \begin{equation*}
        \Delta(x)=\int k(t,x)g(t)\,\mathrm d\sigma(t), \qquad g=\frac{\mathrm d(P-P^*)}{\mathrm d\sigma}, \qquad \sigma=|P|+P^*.
    \end{equation*}
    Since the positive part of $g$ is supported on at most $A$ distinct locations, we must have $V(g)\leq 2A$.
    Moreover, since $\bar F(\boldsymbol{x})>0$, the witness $\Delta$ is not identically zero.
    The variation diminishing property (cf. Def.~\ref{def: Variation deimishing}) therefore gives $Z(\Delta)\leq V(g)$, where zeros are counted with multiplicity.
    On the other hand, the stationarity conditions imply that each $\vartheta_j$ is a zero of multiplicity at least two.
    Consequently,
    \begin{equation*}
        2A\leq Z(\Delta)\leq V(g)\leq 2A.
    \end{equation*}
    Thus all inequalities are equalities, and the $\vartheta_j$ are the only zeros of $\Delta$, each with multiplicity exactly two.
    Since the positive part of $g$ is supported on at most $A$ locations, equality $V(g)=2A$ requires the alternating sign pattern negative--positive--negative--$\cdots$--positive--negative.
    Thus every location in the reduced measure must have strictly positive total weight.
    Moreover, since the negative contributions outside the positive locations must arise from $P^*$, these locations lie within the convex hull of $\supp(P^*)$.

    We next establish that $\Delta\leq0$.
    Since all zeros of $\Delta$ have multiplicity exactly two, $\Delta$ cannot change sign.
    Furthermore, the weight optimality conditions give $\int\Delta\,\mathrm dP=0$, and hence
    \begin{equation*}
        \int\Delta\,\mathrm dP^*=-\int\Delta\,\mathrm d(P-P^*)=-\|\Delta\|_{\mathcal H_k}^2<0.
    \end{equation*}
    It follows that $\Delta\leq0$ everywhere.
    Finally, weight optimality gives $\Delta(x_i)=0$ for every particle, including those with zero individual weight.
    Since all zeros of $\Delta$ occur at the active locations $\vartheta_1,\ldots,\vartheta_A$, every particle must coincide with one of these locations.
    Thus every distinct particle location is active, has positive total weight, and lies within the convex hull of $\supp(P^*)$.
    Moreover, each particle location is a strict local maximum of $\Delta$, and its zero has multiplicity exactly two.
    Consequently, for every $i=1,\ldots,N$,
    \begin{equation*}
        \Delta'(x_i)=0, \qquad \Delta''(x_i)<0. \qedhere
    \end{equation*}
\end{proof}

\posweightrelaxed*
\begin{proof}
    Suppose $w_i=0$.
    Since we are considering a local minimum with $\bar F(\boldsymbol{x})>0$, we have $\|\Delta\|_{\mathcal H_k}^2=\bar F(\boldsymbol{x})>0$.
    Thus, $\Delta$ is not identically zero.
    It follows that, for every $\varepsilon>0$, there exists an $x'$ in an $\varepsilon$-neighbourhood of $x_i$ such that $x'\neq x_j$ for all $j=1,\ldots,N$ and $\Delta(x')\neq0$.
    This follows from analyticity when the kernel is analytic; alternatively, when the kernel is ETP and the problem is well specified, $\Delta$ is a nonzero linear combination of finitely many kernel sections and therefore has only finitely many zeros.

    Fix such a point $x'$.
    Since $w_i=0$, moving the $i$th particle to $x'$ leaves the measure unchanged if its weight remains zero.
    Reoptimising the weights at the new configuration can only decrease the objective relative to keeping the other weights fixed and optimising only the weight of the moved particle.
    Thus, using the reproducing property, we have that
    \begin{align*}
        \bar F(x_1,\ldots,x_{i-1},x',x_{i+1},\ldots,x_N) &\leq \min_{w_i'\in\mathbb R}\|\Delta+w_i'k(x',\cdot)\|_{\mathcal H_k}^2\\
        &=\bar F(\boldsymbol{x})+\min_{w_i'\in\mathbb R}\left\{2w_i'\Delta(x')+(w_i')^2k(x',x')\right\}\\
        &=\bar F(\boldsymbol{x})-\frac{\Delta(x')^2}{k(x',x')} <\bar F(\boldsymbol{x}).
    \end{align*}
    Since $x'$ can be chosen arbitrarily close to $x_i$, this contradicts local minimality.
    Finally, suppose two particle locations coincide, say $x_i=x_j$.
    We can then transfer the weight of particle $i$ to particle $j$, setting $w_j\leftarrow w_j+w_i$ and $w_i\leftarrow0$, without changing the measure or the objective.
    This yields an optimal weight vector with a zero entry, contradicting the preceding argument.
\end{proof}

\runstablecollisionpoints*
\begin{proof}
    Define the reduced measure as in the proof of Prop.~\ref{prop: relax max principle}.
    Without loss of generality, suppose the collision occurs at $\vartheta_1$.
    Consider perturbing some, but not all, of the particles positioned at $\vartheta_1$ to $\vartheta_1+d$, for sufficiently small $d>0$.
    Suppose the resulting locations $(\vartheta_1+d,\vartheta_1,\ldots,\vartheta_A)$ have weights $(\varsigma_0^d,\ldots,\varsigma_A^d)$.
    Then
    \begin{equation*}
        \Delta_d(x) = \varsigma^d_0 k(\vartheta_1 + d, x) + \varsigma^d_1 k(\vartheta_1, x) + \sum_{i=2}^A \varsigma^d_i k(\vartheta_i, x) - \mu_{P^*}(x).
    \end{equation*}
    Following \citet{batenkov2013geometrysingularitiespronymapping} and \citet{1f9723a1-a04e-3d21-9c1b-61527608e570}, introduce the basis functions $\phi^d_i(\cdot)=k(\vartheta_i,\cdot)$ for $i=1,\ldots,A$, together with $\phi^d_0=\frac{k(\vartheta_1+d,\cdot)-k(\vartheta_1,\cdot)}{d}$.
    Define $b^d_0=\frac{\mu_{P^*}(\vartheta_1+d)-\mu_{P^*}(\vartheta_1)}{d}$, and set $b^d_i=\mu_{P^*}(\vartheta_i)$ for $i=1,\ldots,A$.
    We can then write
    \begin{equation*}
        \Delta_d(x) = d\varsigma^d_0 \phi^d_0(x) + (\varsigma^d_0 + \varsigma^d_1)\phi_1^d(x) + \sum_{i=2}^A \varsigma^d_i \phi^d_i(x) - \mu_{P^*}(x).
    \end{equation*}
    Now define
    \begin{equation*}
        \mathbf{v}^d = \begin{bmatrix} d\varsigma_0^d\\ \varsigma^d_0 + \varsigma^d_1\\ \varsigma_2^d\\ \vdots\\ \varsigma_A^d \end{bmatrix}.
    \end{equation*}
    This is an equivalent representation of the optimised weights.
    For $d>0$, the map between $\mathbf v^d$ and $\boldsymbol{\varsigma}^d$ is invertible, so optimising either gives the same result.
    Define the matrix $K^d$ by $K^d_{ij}=\langle\phi^d_i,\phi^d_j\rangle_\mathcal{H}$.
    By optimality of the weights, we have $\Delta_d(\vartheta_1+d)=0$ and $\Delta_d(\vartheta_i)=0$ for $i=1,\ldots,A$.
    Consequently,
    \begin{equation*}
        K^d\mathbf{v}^d-\mathbf{b}^d = \begin{bmatrix} \frac{\Delta_d(\vartheta_1+d)-\Delta_d(\vartheta_1)}{d}\\ \Delta_d(\vartheta_1)\\ \vdots\\ \Delta_d(\vartheta_A) \end{bmatrix} = \mathbf{0}.
    \end{equation*}
    Define the limiting basis by $\phi_0^0=\partial_1k(\vartheta_1,\cdot)$, with the remaining basis functions unchanged.
    This also defines $K^0$ and $\mathbf{b}^0$.
    Since $k$ is ETP, it has the smoothness required here, and hence $K^d\to K^0$ and $\mathbf b^d\to\mathbf b^0$ as $d\to0$.
    In this limiting basis, $(0,\boldsymbol{\varsigma})^\top$ represents the original reduced measure, and the reproducing property gives
    \begin{equation*}
        K^0(0,\boldsymbol{\varsigma})^\top-\mathbf{b}^0 = \begin{bmatrix} \Delta'(\vartheta_1)\\ \Delta(\vartheta_1)\\ \vdots\\ \Delta(\vartheta_A) \end{bmatrix} = \mathbf{0},
    \end{equation*}
    where the final equality follows from stationarity at the original configuration.
    Since $K^0$ is invertible by the ETP property, $K^d$ is also invertible for sufficiently small $d$, and hence $\mathbf{v}^d=(K^d)^{-1}\mathbf{b}^d\to(0,\boldsymbol{\varsigma})^\top$ as $d\to0$.

    We next establish continuity of the basis functions and their derivatives as $d\to0$.
    Following \citet{baxter2010functionalsexponentialbrownianmotion}, observe that
    \begin{equation*}
        \phi_0^d(x) = \int_0^1 \partial_1k(\vartheta_1+td,x)\,\mathrm dt.
    \end{equation*}
    By smoothness of the kernel, we can differentiate under the integral to obtain
    \begin{equation*}
        \partial_x^n\phi_0^d(x) = \int_0^1 \partial_x^n\partial_1k(\vartheta_1+td,x)\,\mathrm dt.
    \end{equation*}
    Since the integrand is continuous in the relevant arguments, these derivatives extend continuously to $d=0$.
    The remaining basis functions are independent of $d$.
    It follows that
    \begin{equation*}
        \Delta_d''(x) = \sum_{i=0}^A v_i^d\partial_x^2\phi_i^d(x)-\mu_{P^*}''(x) \to \Delta''(x),
    \end{equation*}
    and the extension defined by $\Delta_0''(x)=\Delta''(x)$ is jointly continuous in $d$ and $x$.
    This provides the continuity needed for the instability argument.
    Finally, by the mean value theorem,
    \begin{equation*}
        \dot d = -\bigl(\Delta_d'(\vartheta_1+d)-\Delta_d'(\vartheta_1)\bigr) = -\Delta_d''(\varrho)d,
    \end{equation*}
    for some $\varrho\in(\vartheta_1,\vartheta_1+d)$.
    By continuity, $\Delta_d''(\varrho)<0$ in a sufficiently small neighbourhood of the stationary configuration along the perturbation considered above.
    Thus $\dot d$ has the same sign as $d$, establishing instability of the collision.
    Note, in particular, that any collided particles remain collided under the flow.
    Thus,  along the perturbation we retain the negativity when the other particles change positions too, and for a sufficiently small neighbourhood no other collisions occur.
\end{proof}

\relaxepsbound*
\begin{proof}
    Fix a target distribution $P^*$ with compact support in $[A,B]$, where $B-A\le D$.
    By Prop.~\ref{prop: relax max principle}, all particle locations lie within this interval and have positive total weight.
    Moreover, if $\MMD\ge\varepsilon$, then there exists $x_0\in[A,B]$ such that $|\Delta(x_0)|\ge\varepsilon^2$.
    Indeed, stationarity gives $\int\Delta(x)\,\mathrm dP(x)=0$, and hence $\varepsilon^2\le\MMD^2=-\int\Delta(x)\,\mathrm dP^*(x)$.
    The claim then follows since $P^*$ is a probability measure.

    We next bound the analytic continuation of $\Delta$ to the complex plane, centred at $x_0$, following the approach of \citet[Lemma 4.3]{gröchenig2018sharpresultssamplingderivatives}.
    By Prop.~\ref{prop: relax max principle}, $P=|P|$, and hence $|P-P^*|\le P+P^*$.
    Moreover, by the same proposition we have $\Delta=\mu_P-\mu_{P^*}\le0$, so $\mu_P\le\mu_{P^*}\le1$, since $P^*$ is a probability measure and the kernel is bounded above by one.
    It follows that,
    \begin{align*}
        |\Delta(x_0+x+iy)| &= \left|\int e^{-(x_0+x+iy-t)^2}\,\mathrm d(P-P^*)(t)\right|\\
        &\le e^{y^2}\int e^{-(x_0+x-t)^2}\,\mathrm d(P+P^*)(t)
        \le e^{y^2}\bigl(\mu_P(x+x_0)+\mu_{P^*}(x+x_0)\bigr)\le 2e^{y^2}.
    \end{align*}
    We can therefore apply Prop.~\ref{prop: jenzerocount} with $f=\Delta(\cdot+x_0)$, $A_f=2$, and $B_f=1$.
    In particular, since each distinct particle location contributes at least two zeros, counted with multiplicity, we obtain
    \begin{equation*}
        N' \le \frac{\log(\frac{2}{\varepsilon^2}) + \frac{1}{2}\delta^2 (B-A)^2}{2\log(\delta)} \le \frac{\log(\frac{2}{\varepsilon^2}) + \frac{1}{2}\delta^2 D^2}{2\log(\delta)},
    \end{equation*}
    where $N'$ denotes the number of distinct particle locations.
    If $N$ exceeds this bound, then at least two particles must share a location.
    By Props.~\ref{prop: runstablecollisionpoints} and~\ref{prop: posweightrelaxed}, such configurations cannot be local minima and are unstable under the unweighted particle gradient flow.
\end{proof}

Optimising the bound above, we obtain an asymptotic rate as $\varepsilon \to  0$ of $O(\frac{\log(1/\varepsilon)}{\log \log(1/\varepsilon)})$.
We can see this by setting, for sufficiently small $\varepsilon$, $\delta = \sqrt{-\log(\varepsilon)}$, which gives
\begin{equation*}
    N' \le \frac{\log 2 + 2\log(1/\varepsilon) + 0.5D^2\log(1/\varepsilon)}{\log \log (1/\varepsilon)} = O(\frac{\log(1/\varepsilon)}{\log \log(1/\varepsilon)}).
\end{equation*}

\subsection{Simplex Variant}
\Hinterp*

\begin{proof}
    First, the interpolation conditions follow by considering the matrix
    \begin{equation*}
        M(\boldsymbol{x^*},\boldsymbol{\bar{x}}) = \begin{bmatrix} k(\boldsymbol{x^*},\boldsymbol{\bar{x}}) & \partial_2k(\boldsymbol{x^*},\boldsymbol{\bar{x}})\\ \partial_1k(\boldsymbol{x^*},\boldsymbol{\bar{x}}) & \partial_1\partial_2k(\boldsymbol{x^*},\boldsymbol{\bar{x}}) \end{bmatrix}.
    \end{equation*}
    In particular, it suffices to find coefficients $\boldsymbol{\alpha},\boldsymbol{\beta}$ satisfying
    \begin{equation*}
        M(\boldsymbol{x^*},\boldsymbol{\bar{x}})\begin{bmatrix} \boldsymbol{\alpha}\\ \boldsymbol{\beta} \end{bmatrix} = \begin{bmatrix} \mathbf{1}\\ \mathbf{0} \end{bmatrix}.
    \end{equation*}
    Since the Gaussian kernel is ETP, the matrix $M(\boldsymbol{x^*},\boldsymbol{\bar{x}})$ is invertible, and so such coefficients exist.
    It remains now to establish that $H(x)<1$ away from the target locations.
    By the interpolation conditions, $H-1$ has at least $2K$ zeros, counted with multiplicity.
    Multiplication by $\smash{-e^{x^2}}$ preserves these zeros and their multiplicities.
    Define
    \begin{equation*}
        \tilde\alpha_i=(-\alpha_i+2\beta_i\bar{x}_i)e^{-\bar{x}_i^2}, \qquad \tilde\beta_i=-2\beta_i e^{-\bar{x}_i^2}, \qquad s_i=2\bar{x}_i,
    \end{equation*}
    so that
    \begin{equation*}
        -e^{x^2}(H(x)-1)=\sum_{i=1}^K(\tilde\alpha_i+\tilde\beta_i x)e^{s_i x}+e^{x^2}.
    \end{equation*}
    By Lemma 2.6 of \citet{schiebinger2015superresolutionseparation}, this function, and hence also $H-1$, has at most $2K$ zeros, counted with multiplicity.
    Thus, the prescribed target locations are the only zeros of $H-1$, each with multiplicity exactly two.
    Since $H(x)\to0$ as $x\to\pm\infty$, and $H-1$ cannot change sign at these zeros, it follows that $H(x)<1$ everywhere else.

    We remark that an alternative interpretation of Lemma 2.6 of \citet{schiebinger2015superresolutionseparation} is that the Gaussian kernel, together with its derivatives and the constant function, forms a Chebyshev system.
\end{proof}

\cpos*
\begin{proof}
    By Prop.~\ref{prop: Hinterp}, we can construct coefficients $\boldsymbol{\alpha}',\boldsymbol{\beta}'$ such that the function
    $
    H(\cdot) = \sum_{i=1}^A (\alpha'_i k(\cdot,\vartheta_i)+\beta'_i\partial_2k(\cdot,\vartheta_i))
    $
    satisfies $H(\vartheta_i) = 1$, $H'(\vartheta_i) = 0$ for all $i=1,\ldots,A$, and $H(x) < 1$ for all $x\notin\{\vartheta_1,\ldots,\vartheta_A\}$.

    At a stationary point with $F(\mathbf w,\boldsymbol{x})>0$, we must have $\supp(P^*)\not\subseteq\supp(P)$.
    Indeed, if this inclusion held, the target could be represented exactly at the current particle locations, and convexity of the weight optimisation problem would imply that the stationary weights achieve $F(\mathbf w,\boldsymbol{x})=0$.
    Since $H$ equals one at every active particle location and is strictly less than one elsewhere, it follows that
    \begin{align*}
        \int H(x)\,\mathrm dP(x) = \sum_{j=1}^A \varsigma_jH(\vartheta_j)=1, \qquad
        \int H(x)\,\mathrm dP^*(x) <1.
    \end{align*}
    Hence, in particular, 
    \begin{equation}
        \label{Positive test}
        \int H(x)\,\mathrm d(P-P^*)(x)>0.
    \end{equation}
    Meanwhile, the stationarity conditions \eqref{Stationary0redsimp} give
    \begin{align*}
        \int H(x)\,\mathrm d(P-P^*)(x) = \sum_{i=1}^A \left(\alpha'_i\Delta(\vartheta_i)+\beta'_i\Delta'(\vartheta_i)\right)
        = c\sum_{i=1}^A\alpha'_i.
    \end{align*}
    Moreover, by the reproducing property and the interpolation conditions,
    \begin{align*}
        \|H\|_{\mathcal H}^2 &= \begin{bmatrix} (\boldsymbol{\alpha}')^\top & (\boldsymbol{\beta}')^\top \end{bmatrix}M(\boldsymbol{\vartheta},\boldsymbol{\vartheta})\begin{bmatrix} \boldsymbol{\alpha}'\\ \boldsymbol{\beta}' \end{bmatrix}
        = \sum_{i=1}^A\alpha'_i>0,
    \end{align*}
    where $M(\cdot,\cdot)$ is as defined in the proof of Prop.~\ref{prop: Hinterp}.
    Combining these identities with \eqref{Positive test} gives $c>0$.
    Finally, since the Gaussian kernel is strictly positive and bounded above by one, we have $\mu_{P^*}(\vartheta_i)>0$ and $\mu_P(\vartheta_i)\le1$.
    Thus $c=\mu_P(\vartheta_i)-\mu_{P^*}(\vartheta_i)<\mu_P(\vartheta_i)\le1$, establishing $0<c<1$.
\end{proof}

\simpmaxprinciple*
\begin{proof}
    As before, we begin by defining the reduced measure of $P$ by merging particles at the same location and removing locations with zero total weight.
    Write $\vartheta_1,\ldots,\vartheta_A$ for its distinct locations and $\varsigma_1,\ldots,\varsigma_A$ for their nonzero total weights.
    By the stationarity conditions, we have that
    \begin{align}
        \Delta(\vartheta_j)=c, \qquad \Delta'(\vartheta_j)=0, \qquad j=1,\ldots,A.
    \end{align}
    Let $h=\Delta-c$. 
    To apply the variation diminishing property to this function, we first represent the constant $c$ as a kernel integral using translation invariance and normalisability.
    Let $\lambda$ denote Lebesgue measure and $\mathcal Z=\int_{\mathbb R}k(t,x)\,\mathrm dt$, which is finite, positive, and independent of $x$.
    If $c=0$, write
    \begin{equation*}
        h(x)=\int k(t,x)g(t)\,\mathrm d\sigma(t), \qquad g=\frac{\mathrm d(P-P^*)}{\mathrm d\sigma}, \qquad \sigma=|P|+P^*.
    \end{equation*}
    If $c>0$, instead write
    \begin{equation*}
        h(x)=\int k(t,x)g(t)\,\mathrm d\sigma(t), \qquad g=\frac{\mathrm d(P-P^*-\frac{c}{\mathcal Z}\lambda)}{\mathrm d\sigma}, \qquad \sigma=|P|+P^*+\lambda.
    \end{equation*}
    Since the positive part of $g$ is supported on at most $A$ distinct locations, we must have $V(g)\leq 2A$.
    Moreover, since $F(\mathbf w,\boldsymbol{x})>0$ and $P,P^*$ have equal mass, we have that $\int h\,\mathrm d(P-P^*)=\|\Delta\|_{\mathcal H_k}^2>0$, so $h$ is not identically zero.
    The variation diminishing property (cf. Def.~\ref{def: Variation deimishing}) therefore gives $Z(h)\leq V(g)$, where zeros are counted with multiplicity.
    On the other hand, the stationarity conditions imply that each $\vartheta_j$ is a zero of multiplicity at least two.
    Consequently,
    \begin{equation*}
        2A\leq Z(h)\leq V(g)\leq 2A.
    \end{equation*}
    Thus, all inequalities are in fact equalities, and the $\vartheta_j$ are the only zeros of $h$, each with multiplicity exactly two.
    Since the positive part of $g$ is supported on at most $A$ locations, equality $V(g)=2A$ requires the alternating sign pattern negative--positive--negative--$\cdots$--positive--negative.
    Thus, every location in the reduced measure must have strictly positive total weight.

    We next establish that $h\leq0$.
    Since all zeros of $h$ have multiplicity exactly two, $h$ cannot change sign.
    Furthermore, the stationarity conditions give $\int h\,\mathrm dP=0$, and therefore
    \begin{equation*}
        \int h\,\mathrm dP^*=-\int h\,\mathrm d(P-P^*)=-\|\Delta\|_{\mathcal H_k}^2<0.
    \end{equation*}
    It follows that $h\leq0$ everywhere.
    Finally, the KKT conditions give $h(x_i)=\Delta(x_i)-c\geq0$ for every particle, including those with zero individual weight.
    Since $h\leq0$, we must thus have $h(x_i)=0$.
    Since all zeros of $h$ occur at the active locations $\vartheta_1,\ldots,\vartheta_A$, every particle must coincide with one of these locations.
    Hence every distinct particle location is active and has positive total weight.
    Moreover, each particle location is a strict local maximum of $h$, and its zero has multiplicity exactly two.
    Thus, since $h=\Delta-c$, it follows that, for every $i=1,\ldots,N$,
    \begin{equation*}
        \Delta(x_i)=c, \qquad \Delta'(x_i)=0, \qquad \Delta''(x_i)<0. \qedhere
    \end{equation*}
\end{proof}

\neighbourhoodmiss*

\begin{proof}
    First, if $K>1$, it suffices to consider configurations satisfying $\max_j\min_i|x_i-x_j^*|<\frac12\min_{i\neq j}|x_i^*-x_j^*|$, since otherwise the desired separation bound already holds.
    For each target location $x_i^*$, let $\bar x_i$ denote a nearest particle location.
    The separation condition ensures that the $\bar x_i$ are distinct.
    By Prop.~\ref{prop: Hinterp}, we can construct coefficients $\boldsymbol{\alpha},\boldsymbol{\beta}$ such that the function
    $H(\cdot)=\sum_{i=1}^K(\alpha_i k(\cdot,\bar x_i)+\beta_i\partial_2k(\cdot,\bar x_i))$
    satisfies $H(x_i^*)=1$, $H'(x_i^*)=0$ for all $i=1,\ldots,K$, and $H(x)<1$ for all $x\notin\{x_1^*,\ldots,x_K^*\}$.
    Then, as in the proof of Prop.~\ref{prop: cpos}, we have
    \begin{equation*}
        \int H(x)\,\mathrm dP(x)=\sum_{j=1}^N w_jH(x_j)\leq1, \qquad \int H(x)\,\mathrm dP^*(x)=\sum_{i=1}^K w_i^*H(x_i^*)=1.
    \end{equation*}
    Hence, at a spurious stationary point,
    \begin{equation}
        \label{Negative test}
        \int H(x)\,\mathrm d(P-P^*)(x)\leq0.
    \end{equation}
    Meanwhile, Prop.~\ref{prop: simp max principle} gives $\Delta(\bar x_i)=c$ and $\Delta'(\bar x_i)=0$, and therefore
    \begin{equation*}
        \int H(x)\,\mathrm d(P-P^*)(x)=\sum_{i=1}^K\left(\alpha_i\Delta(\bar x_i)+\beta_i\Delta'(\bar x_i)\right)=c\sum_{i=1}^K\alpha_i.
    \end{equation*}
    Since $c>0$ by Prop.~\ref{prop: cpos}, it suffices to establish that $\sum_{i=1}^K\alpha_i>0$ whenever $\boldsymbol{\bar x}$ is sufficiently close to $\boldsymbol{x^*}$.
    At $\boldsymbol{\bar x}=\boldsymbol{x^*}$, the reproducing property and interpolation conditions give
    \begin{equation*}
        \|H\|_{\mathcal H}^2=\begin{bmatrix}\boldsymbol{\alpha}^\top & \boldsymbol{\beta}^\top\end{bmatrix}M(\boldsymbol{x^*},\boldsymbol{x^*})\begin{bmatrix}\boldsymbol{\alpha}\\\boldsymbol{\beta}\end{bmatrix}=\sum_{i=1}^K\alpha_i>0,
    \end{equation*}
    where $M(\cdot,\cdot)$ is as defined in the proof of Prop.~\ref{prop: Hinterp}.
    By continuity of the interpolation coefficients, this positivity persists in a sufficiently small target-dependent neighbourhood of $\boldsymbol{x^*}$.
    Within this neighbourhood, we therefore obtain a contradiction to \eqref{Negative test}.
    Consequently, no spurious stationary point can lie in this neighbourhood.

    We now make the neighbourhood explicit using standard matrix perturbation bounds.
    By the interpolation conditions,
    \begin{equation*}
        \begin{bmatrix}\boldsymbol{\alpha}\\\boldsymbol{\beta}\end{bmatrix}=M(\boldsymbol{x^*},\boldsymbol{\bar x})^{-1}\begin{bmatrix}\mathbf 1\\\mathbf 0\end{bmatrix},
    \end{equation*}
    where
    \begin{equation*}
        M(\boldsymbol{x^*},\boldsymbol{\bar x})=\begin{bmatrix}k(\boldsymbol{x^*},\boldsymbol{\bar x}) & \partial_2k(\boldsymbol{x^*},\boldsymbol{\bar x})\\ \partial_1k(\boldsymbol{x^*},\boldsymbol{\bar x}) & \partial_1\partial_2k(\boldsymbol{x^*},\boldsymbol{\bar x})\end{bmatrix}.
    \end{equation*}
    Thus, writing $e_1=\begin{bmatrix}\mathbf 1\\\mathbf 0\end{bmatrix}$, we have that
    \begin{equation}
        T(\boldsymbol{\bar x})=\sum_{i=1}^K\alpha_i=e_1^\top M(\boldsymbol{x^*},\boldsymbol{\bar x})^{-1}e_1.
    \end{equation}
    We seek an explicit radius $r(P^*)>0$ such that $T(\boldsymbol{\bar x})>0$ whenever $\delta(\boldsymbol{\bar x})=\|\boldsymbol{\bar x}-\boldsymbol{x^*}\|_\infty<r(P^*)$.
    For brevity, write $M(\boldsymbol{\bar x})=M(\boldsymbol{x^*},\boldsymbol{\bar x})$, $M^*=M(\boldsymbol{x^*},\boldsymbol{x^*})$, and $T^*=T(\boldsymbol{x^*})$.
    Notice that $T^*\geq1$.
    Define $E(\boldsymbol{\bar x})=M(\boldsymbol{\bar x})-M^*$ and $q(\boldsymbol{\bar x})=\|(M^*)^{-1}E(\boldsymbol{\bar x})\|_\infty$.
    By the Neumann series, whenever $q(\boldsymbol{\bar x})<1$,
    \begin{equation}
        \|(M^*)^{-1}-M(\boldsymbol{\bar x})^{-1}\|_\infty\leq\frac{q(\boldsymbol{\bar x})}{1-q(\boldsymbol{\bar x})}\|(M^*)^{-1}\|_\infty.
    \end{equation}
    Let $m^{-1}=\|(M^*)^{-1}\|_\infty$.
    Then $q(\boldsymbol{\bar x})\leq m^{-1}\|E(\boldsymbol{\bar x})\|_\infty$.
    By standard derivative bounds for the Gaussian kernel and the mean value theorem,
    \begin{equation*}
        \|E(\boldsymbol{\bar x})\|_\infty\leq6K\delta(\boldsymbol{\bar x}).
    \end{equation*}
    Consequently, Hölder's inequality gives
    \begin{align*}
        |T^*-T(\boldsymbol{\bar x})| &\leq\|e_1\|_1\|(M^*)^{-1}-M(\boldsymbol{\bar x})^{-1}\|_\infty\|e_1\|_\infty
        \leq\frac{6m^{-2}K^2\delta(\boldsymbol{\bar x})}{1-6m^{-1}K\delta(\boldsymbol{\bar x})}.
    \end{align*}
    Since $T^*\geq1$, it suffices to ensure that the final bound is strictly less than one.
    Rearranging the above, we have that
    \begin{equation}
        \delta(\boldsymbol{\bar x})<\frac{1}{6m^{-1}K(m^{-1}K+1)},
    \end{equation}
    which also guarantees $q(\boldsymbol{\bar x})<1$.
    Thus, we can take any $r(P^*)>0$ satisfying
    \begin{equation*}
        r(P^*)<\min\left\{\frac12\min_{i\neq j}|x_i^*-x_j^*|,\frac{1}{6m^{-1}K(m^{-1}K+1)}\right\},
    \end{equation*}
    where the first term is interpreted as $+\infty$ whenever $K=1$.
\end{proof}

\MMDFourierSplit*

\begin{proof}
    Let $\rho=P-P^*$ and $\Delta=\mu_P-\mu_{P^*}$.
    By the definition of the MMD, we have that
    \begin{equation*}
        \MMD^2(P,P^*)=\iint k(x,y)\,\mathrm d\rho(x)\,\mathrm d\rho(y)=\int\Delta(x)\,\mathrm d\rho(x).
    \end{equation*}
    Meanwhile, by Bochner's theorem, it follows that
    \begin{equation*}
        \MMD^2(P,P^*)=\frac{1}{(2\pi)^d}\int\hat k(\omega)|\hat\rho(\omega)|^2\,\mathrm d\omega, \qquad \frac{|\hat\Delta(\omega)|^2}{\hat k(\omega)}=\hat k(\omega)|\hat\rho(\omega)|^2.
    \end{equation*}
    Splitting the integral at frequency $\nu$, and using $|\hat\Delta(\omega)|\leq\|\Delta\|_1$ and $|\hat\rho(\omega)|\leq\|\rho\|_{\mathrm{TV}}$, gives
    \begin{align*}
        \MMD^2(P,P^*) &= \frac{1}{(2\pi)^d}\int_{|\omega|\leq\nu}\frac{|\hat\Delta(\omega)|^2}{\hat k(\omega)}\,\mathrm d\omega+\frac{1}{(2\pi)^d}\int_{|\omega|>\nu}\hat k(\omega)|\hat\rho(\omega)|^2\,\mathrm d\omega\\
        &\leq\frac{\|\Delta\|_1^2}{(2\pi)^d}\int_{|\omega|\leq\nu}\frac{1}{\hat k(\omega)}\,\mathrm d\omega+\frac{\|\rho\|_{\mathrm{TV}}^2}{(2\pi)^d}\int_{|\omega|>\nu}\hat k(\omega)\,\mathrm d\omega.
    \end{align*}
    For the one-dimensional Gaussian kernel with $2\sigma^2=1$, we have $\hat k(\omega)=\sqrt{\pi}e^{-\omega^2/4}$.
    Moreover, since $P$ and $P^*$ are probability measures, $\|\rho\|_{\mathrm{TV}}^2\leq4$.
    Substituting into the preceding bound gives
    \begin{align*}
        \MMD^2(P,P^*) &\leq\frac{\|\Delta\|_1^2}{2\pi\sqrt{\pi}}\int_{|\omega|\leq\nu}e^{\omega^2/4}\,\mathrm d\omega+\frac{2}{\sqrt{\pi}}\int_{|\omega|>\nu}e^{-\omega^2/4}\,\mathrm d\omega
        =\frac{\|\Delta\|_1^2}{\pi}\erfi\left(\frac{\nu}{2}\right)+4\erfc\left(\frac{\nu}{2}\right). \qedhere
    \end{align*}
\end{proof}

\Lcbound*
\begin{proof}
    \label{proof of c L1 bound}
    Since $\int_{\mathbb R}\Delta(x)\,\mathrm dx=0$, the positive and negative parts of $\Delta$ have equal integrals, and so
    \begin{equation}
        \label{eq:posdouble}
        \|\Delta\|_1=2\int_{-\infty}^\infty[\Delta]_+.
    \end{equation}
    Write $\mu_P(\cdot)=\sum_i w_i e^{-(x_i-\cdot)^2}$ and $\mu_{P^*}(\cdot)=\int e^{-(x^*-\cdot)^2}\,\mathrm dP^*(x^*)$.
    Let $L,R$ denote the leftmost and rightmost particle locations, and $A,B$ the endpoints of the target support.
    We split the integral above into seven regions:
    \begin{equation}
        (-\infty,L], \qquad [L, L+\delta_L], \qquad [L+\delta_L, A-\varepsilon_L], \qquad [A-\varepsilon_L, B+\varepsilon_R], \qquad [B+\varepsilon_R, R-\delta_R], \qquad [R-\delta_R, R], \qquad [R, \infty),
    \end{equation}
    treating any reversed intervals as empty.
    We impose $\delta_L,\delta_R\geq1/\sqrt{2}$ and define $D_L=A-L$ and $D_R=R-B$.
    If $D_L\leq\delta_L$, we set $\varepsilon_L=0$ and omit the contribution from $[L+\delta_L,A]$, and similarly on the right if $D_R\leq\delta_R$.
    We proceed as follows
    \begin{enumerate}[label={},leftmargin=2em]
        \item $(-\infty,L]$:
        By the maximum principle and the Gaussian tail bound,
        \begin{align*}
            \int_{-\infty}^L[\Delta]_+ &\leq\int_{-\infty}^L\min\{c,\mu_P(x)\}\,\mathrm dx
            \leq\int_{-\infty}^0\min\{c,e^{-x^2}\}\,\mathrm dx
            =c\sqrt{-\log c}+\frac{\sqrt{\pi}}{2}\erfc(\sqrt{-\log c}).
        \end{align*}

        \item $[L,L+\delta_L]$:
        Applying the maximum principle (cf. Prop.~\ref{prop: simp max principle}) gives
        \begin{equation*}
            \int_L^{L+\delta_L}[\Delta]_+\leq c\delta_L.
        \end{equation*}

        \item $[L+\delta_L,A-\varepsilon_L]$:
        Let $X\sim P$ and define $f(x)=xe^{-x^2}$.
        Since $f$ has a single maximum, we have
        \begin{align*}
            \mathbb P(L+\delta_L\leq X\leq A-\varepsilon_L) &\leq\frac{\mathbb E[f(X-L)]}{\min\{f(\delta_L),f(A-\varepsilon_L-L)\}}\\
            &=\frac{\sum_iw_if(x_i-L)}{\min\{f(\delta_L),f(A-\varepsilon_L-L)\}}=\frac{\int f(x^*-L)\,\mathrm dP^*(x^*)}{\min\{f(\delta_L),f(A-\varepsilon_L-L)\}},
        \end{align*}
        where the final equality follows from stationarity at $L$.
        Since $D_L\geq\delta_L\geq1/\sqrt{2}$, the function $f$ is decreasing on $[D_L,\infty)$, and hence $f(x^*-L)\leq f(D_L)$ for every $x^*\in[A,B]$.
        Moreover, since $D_L-\varepsilon_L\geq\delta_L$, we have $f(D_L-\varepsilon_L)\leq f(\delta_L)$.
        Consequently,
        \begin{equation*}
            \mathbb P(L+\delta_L\leq X\leq A-\varepsilon_L)\leq\frac{D_L}{D_L-\varepsilon_L}\exp(-2\varepsilon_LD_L+\varepsilon_L^2).
        \end{equation*}
        It follows that
        \begin{align*}
            \int_{L+\delta_L}^{A-\varepsilon_L}[\Delta]_+ \leq \frac{D_L}{D_L-\varepsilon_L}\exp(-2\varepsilon_LD_L+\varepsilon_L^2)\sqrt{\pi}+2c\sqrt{-\log c}+\sqrt{\pi}\erfc(\sqrt{-\log c}),
        \end{align*}
        where the final two terms account for contributions from particles outside the interval.
        These are obtained by separating the contributions to $\mu_P$ from particles inside and outside the interval, and applying the preceding Gaussian tail estimate to the latter.

        \item $[A-\varepsilon_L,B+\varepsilon_R]$:
        Applying the maximum principle (cf. Prop.~\ref{prop: simp max principle}) again gives
        \begin{equation*}
            \int_{A-\varepsilon_L}^{B+\varepsilon_R}[\Delta]_+\leq c(\varepsilon_L+\varepsilon_R+B-A).
        \end{equation*}
    \end{enumerate}
    The remaining intervals $[B+\varepsilon_R,R-\delta_R]$, $[R-\delta_R,R]$, and $[R,\infty)$ are bounded symmetrically.
    Combining these estimates yields
    \begin{equation}
        \begin{aligned}
            \int_{-\infty}^\infty[\Delta]_+ \leq{}&c(\delta_L+\delta_R+B-A)
            +6\left(c\sqrt{-\log c}+\frac{\sqrt{\pi}}{2}\erfc(\sqrt{-\log c})\right)\\
            &+c\varepsilon_L+\frac{D_L}{D_L-\varepsilon_L}\exp(-2\varepsilon_LD_L+\varepsilon_L^2)\sqrt{\pi}\\
            &+c\varepsilon_R+\frac{D_R}{D_R-\varepsilon_R}\exp(-2\varepsilon_RD_R+\varepsilon_R^2)\sqrt{\pi}.
        \end{aligned}
    \end{equation}
    We now choose $\delta_L=\delta_R=1/\sqrt{2}$ and bound the terms involving $\varepsilon_L$ and $\varepsilon_R$.
    Consider first
    \begin{equation*}
        c\varepsilon_L+\frac{D_L}{D_L-\varepsilon_L}\exp(-2\varepsilon_LD_L+\varepsilon_L^2)\sqrt{\pi}.
    \end{equation*}
    If $D_L^2\leq\frac12-\log c$, we can bound the contribution from $[L+\delta_L,A]$ directly using the maximum principle.
    Since this interval has length $D_L-1/\sqrt{2}$ when nonempty, its contribution is bounded by
    \begin{equation*}
        c\left(D_L-\frac{1}{\sqrt{2}}\right)\leq c\sqrt{-\log c}.
    \end{equation*}
    Otherwise, choose $\varepsilon_L=D_L-\sqrt{D_L^2+\log c}$.
    Then $-2\varepsilon_LD_L+\varepsilon_L^2=\log c$, and $\varepsilon_L\leq\sqrt{-\log c}$.
    Moreover, the ratio
    \begin{equation*}
        \frac{D_L}{D_L-\varepsilon_L}=\frac{D_L}{\sqrt{D_L^2+\log c}}
    \end{equation*}
    is decreasing in $D_L$, and is therefore bounded by its value at $D_L^2=\frac12-\log c$, namely $\sqrt{1-2\log c}$.
    Thus, in either case, the corresponding contribution is bounded by
    \begin{equation*}
        c\sqrt{-\log c}+\sqrt{\pi}c\sqrt{1-2\log c}.
    \end{equation*}
    Applying the same argument to the $\varepsilon_R$ term, and using $B-A\leq D$, gives
    \begin{equation*}
        \int_{-\infty}^\infty[\Delta]_+\leq(\sqrt{2}+D)c+8c\sqrt{-\log c}+3\sqrt{\pi}\erfc(\sqrt{-\log c})+2\sqrt{\pi}c\sqrt{1-2\log c}.
    \end{equation*}
    Combining this with \eqref{eq:posdouble} gives the stated bound.
\end{proof}

\MMDHoleBound*
\begin{proof}
    We begin with the variational characterisation of the MMD, which gives, for any nonzero $q\in\mathcal H$,
    \begin{equation*}
        \MMD(P,P^*)\geq\frac{\left|\sum_iw_iq(x_i)-\sum_jw_j^*q(x_j^*)\right|}{\|q\|_{\mathcal H}}.
    \end{equation*}
    Motivated by related constructions in \citet{decastro2020supermixsparseregularizationmixtures} and \citet{Costarelli2021}, we consider a Jackson-type kernel centred at a target location $x_j^*$,
    \begin{equation*}
        q_j(x)=\left(\frac{\sin(a(x-x_j^*))}{a(x-x_j^*)}\right)^{2m}.
    \end{equation*}
    We first bound its RKHS norm.
    Writing $\hat q_j$ and $\hat k$ for the Fourier transforms of $q_j$ and the Gaussian kernel, respectively, we have that
    \begin{equation*}
        \|q_j\|_{\mathcal H}^2=\frac{1}{2\pi}\int_{\mathbb R}\frac{|\hat q_j(\omega)|^2}{\hat k(\omega)}\,\mathrm d\omega.
    \end{equation*}
    Since $\hat k(\omega)=\sqrt{\pi}e^{-\omega^2/4}$, $\|\hat q_j\|_\infty\leq\pi/a$, and $\hat q_j$ is supported on $[-2ma,2ma]$ as the $2m$-fold convolution of a box function, we then obtain
    \begin{equation*}
        \|q_j\|_{\mathcal H}^2\leq\frac{\pi}{a^2}\erfi(ma)\leq\frac{\pi}{a^2}e^{m^2a^2}.
    \end{equation*}
    Thus $\|q_j\|_{\mathcal H}\leq C(a,m)$, with the bound independent of $j$ by translation invariance of the Gaussian kernel.
    We next use the fact that $q_j(x_j^*)=1$ and $q_j\geq0$ to obtain
    \begin{equation*}
        \sum_k w_k^*q_j(x_k^*)\geq w_j^*.
    \end{equation*}
    Suppose the target atom $x_j^*$ is missed by a neighbourhood of radius $r$.
    Then $q_j(x_i)\leq(ar)^{-2m}$ for every particle location $x_i$, and hence, since the weights sum to one,
    \begin{equation*}
        \sum_iw_iq_j(x_i)\leq(ar)^{-2m}.
    \end{equation*}
    Consequently, letting $w^*=\min_jw_j^*$, due to the variational bound we arrive at
    \begin{equation*}
        \MMD(P,P^*)\geq\frac{[w_j^*-(ar)^{-2m}]_+}{\|q_j\|_{\mathcal H}}\geq\frac{[w^*-(ar)^{-2m}]_+}{C(a,m)}. \qedhere
    \end{equation*}
\end{proof}

\NUpperBound*
\begin{proof}
    Consider an interval of radius $r$ around a particle location $x_i$, and let $p_{i,r}=P([x_i-r,x_i+r])$.
    Since $\mu_P(x_i)=\mu_{P^*}(x_i)+c\geq c$, we have that
    \begin{equation*}
        \mu_P(x_i)=\int e^{-(x_i-y)^2}\,\mathrm dP(y)\leq p_{i,r}+(1-p_{i,r})e^{-r^2}.
    \end{equation*}
    Thus, in particular, 
    \begin{equation*}
        p_{i,r}\geq\frac{c-e^{-r^2}}{1-e^{-r^2}}=:p_r.
    \end{equation*}
    Now place an interval of radius $r$ around each particle location, and let $C_1,\ldots,C_S$ denote the connected components of their union.
    Within each component, construct a maximal $2r$-separated packing by selecting particle locations from left to right, choosing each successive location as soon as it lies more than $2r$ from the preceding one.
    Let $B_s$ denote the number of selected locations in component $C_s$.
    Since the corresponding intervals are disjoint and each has probability mass at least $p_r$, we have
    \begin{equation*}
        1\leq S\leq\sum_{s=1}^S B_s\leq\frac{1}{p_r}.
    \end{equation*}
    Moreover, by maximality of the packing, each component $C_s$ has diameter at most $4B_sr$.
    Let $d\leq r$ satisfy $c>e^{-d^2}$.
    For each component $C_s$, choose $z_0$ at distance $r$ to the left of its leftmost particle location.
    By construction, $z_0$ lies at distance at least $d$ from every particle location.
    Hence $\mu_P(z_0)\leq e^{-d^2}$, and therefore, recalling that $h=\Delta-c$,
    \begin{equation*}
        |h(z_0)|=|c-\Delta(z_0)|\geq c-e^{-d^2}.
    \end{equation*}
    Let $\bar h_{z_0}$ denote the analytic continuation of $h$ to the complex plane, recentred at $z_0$.
    As in the proof of Theorem~\ref{thm: relaxepsbound}, following the approach of \citet{gröchenig2018sharpresultssamplingderivatives}, we have
    \begin{equation*}
        |\bar h_{z_0}(x+iy)|\leq c+2e^{y^2}\leq(c+2)e^{y^2}.
    \end{equation*}
    Now let $\delta>1$ and $R_s\geq(4B_s-1)r$, so that the disc of radius $R_s$ centred at $z_0$ contains all particle locations in $C_s$.
    Applying Prop.~\ref{prop: jenzerocount} with $A_f=c+2$ and $B_f=1$ gives
    \begin{equation*}
        \mathbf n_{\bar h_{z_0}}(R_s)\leq\frac{\log\left(\frac{c+2}{c-e^{-d^2}}\right)+\frac12\delta^2R_s^2}{\log\delta}.
    \end{equation*}
    Since every distinct particle location contributes a zero of multiplicity two, summing over the components gives at least $2A$ zeros, counted with multiplicity.
    Consequently, using the same $\delta$ for every component, we obtain
    \begin{equation*}
        A\leq\frac{S\log\left(\frac{c+2}{c-e^{-d^2}}\right)+\frac12\delta^2\sum_{s=1}^S R_s^2}{2\log\delta}.
    \end{equation*}
    We now choose $\delta=\sqrt e$, $d=r$, $r=\sqrt{\log(2)-\log c}$, and $R_s=(4B_s-1)r$.
    We then have that
    $
        \sum_{s=1}^S R_s^2=r^2\sum_{s=1}^S(4B_s-1)^2\leq r^2(4\sum_{s=1}^S B_s-S)^2\leq r^2(\frac4{p_r}-1)^2.
    $
    With these choices, $p_r=c/(2-c)$, so $S\leq 1/p_r=(2-c)/c$.
    Substituting these bounds into the preceding estimate gives
    \begin{equation*}
        A\leq\frac{2-c}{c}\log\left(\frac{2(c+2)}{c}\right)+\frac e2\left(\frac{8-5c}{c}\right)^2\log\left(\frac2c\right). \qedhere
    \end{equation*}
\end{proof}

\MMDSimplexBound*
\begin{proof}
    By Props.~\ref{prop: neighbourhoodmiss} and \ref{prop: MMDholebound}, there exists a constant $L>0$, depending only on $P^*$, such that every spurious stationary point satisfies $\MMD(P,P^*)\geq L$.
    In Prop.~\ref{prop: MMDFourierSplit}, choose $\nu>0$ such that $4\erfc(\frac{\nu}{2})<\frac{L^2}{2}$.
    Since the bound in Prop.~\ref{prop: Lcbound} tends to zero as $c\to0$, we can choose $\underline c>0$ such that, whenever $c\leq\underline c$, it holds that
    \begin{equation}
        \|\Delta\|_1\leq L\sqrt{\frac{\pi}{2\erfi(\frac{\nu}{2})}}.
    \end{equation}
    By Prop.~\ref{prop: MMDFourierSplit}, this would imply $\MMD(P,P^*)<L$, contradicting the lower bound.
    Thus, every spurious stationary point must satisfy $c>\underline c$.
    Applying Prop.~\ref{prop: NUpperBound} then gives a finite upper bound on the number of distinct active particle locations, depending only on $P^*$.
    If $N$ exceeds this bound, at least two particles must share a location.
    By Props.~\ref{prop: unstablecollisionpoints} and \ref{prop: collisionpoints}, such configurations cannot be local minima and are unstable under the unweighted particle gradient flow.
\end{proof}

\MMDEpsSimplexBound*
\begin{proof}
    Fix $D>0$ and $\varepsilon>0$, and consider any stationary point with $\MMD(P,P^*)\geq\varepsilon$, where $P^*$ has compact support of diameter at most $D$.
    In Prop.~\ref{prop: MMDFourierSplit}, choose $\nu>0$ such that $4\erfc(\frac{\nu}{2})<\frac{\varepsilon^2}{2}$.
    Since the bound in Prop.~\ref{prop: Lcbound} tends to zero as $c\to0$, uniformly over such targets, we can choose $\underline c>0$, depending only on $D$ and $\varepsilon$, such that, whenever $c\leq\underline c$, it holds that
    \begin{equation}
        \|\Delta\|_1\leq\varepsilon\sqrt{\frac{\pi}{2\erfi(\frac{\nu}{2})}}.
    \end{equation}
    By Prop.~\ref{prop: MMDFourierSplit}, this would imply $\MMD(P,P^*)<\varepsilon$, contradicting our assumption.
    Thus every stationary point with $\MMD(P,P^*)\geq\varepsilon$ must satisfy $c>\underline c$.
    Applying Prop.~\ref{prop: NUpperBound} then gives a finite upper bound on the number of distinct active particle locations, depending only on $D$ and $\varepsilon$.
    If $N$ exceeds this bound, at least two particles must share a location.
    By Props.~\ref{prop: unstablecollisionpoints} and \ref{prop: collisionpoints}, such configurations cannot be local minima and are unstable under the unweighted particle gradient flow.
\end{proof}

Optimising the bounds yields an asymptotic particle threshold of $N_\varepsilon(D)=O(\varepsilon^{-4}\log(1/\varepsilon))$ as $\varepsilon\to0$, for fixed $D$.
This is substantially worse than the relaxed-case rate, reflecting the larger number of particles required by our current simplex guarantees.
We briefly justify this rate.
As $x\to\infty$, standard asymptotic estimates \citep[Eq. 7.12.1]{NIST:DLMF} give
\begin{equation*}
    \erfc(x)=O\Big(\frac{e^{-x^2}}{x}\Big), \qquad \erfi(x)=O\Big(\frac{e^{x^2}}{x}\Big).
\end{equation*}
By Prop.~\ref{prop: Lcbound}, we have $\|\Delta\|_1\leq L_1(c)$, where $L_1(c)=O(c\sqrt{-\log c})$ as $c\to0$, for fixed $D$.
Choosing $\nu=2\sqrt{-\log L_1(c)}$ for sufficiently small $c$ in Prop.~\ref{prop: MMDFourierSplit} gives $\MMD^2(P,P^*)=O(c)$.
Thus, whenever $\MMD(P,P^*)\geq\varepsilon$, we must have $c=\Omega(\varepsilon^2)$.
Finally, Prop.~\ref{prop: NUpperBound} bounds the number of distinct active particle locations by $O(c^{-2}\log(1/c))$, yielding the stated threshold $N_\varepsilon(D)=O(\varepsilon^{-4}\log(1/\varepsilon))$.

\section{Preliminary analysis of asymptotic stationarity}
Our main results concern stationary configurations with finite particle locations.
A natural question is whether they extend to asymptotically stationary sequences, in which particle locations may escape to $\pm\infty$.
Such an extension would remove the need for compactness assumptions on the particle locations.
Many of our arguments carry over essentially unchanged, but some require modification.
Below, we outline these extensions and identify an obstruction to extending the finite-position analysis without studying the underlying particle dynamics.
We focus on the simplex-constrained setting.

Throughout, consider a sequence $\{(\mathbf w^n,\boldsymbol{x}^n)\}_{n=1}^\infty$ and constants $c_n$ satisfying, for $j=1,\dots,N$, 
\begin{align}
    [c_n-\Delta_n(x_j^n)]_+\to0, \qquad w_j^n(\Delta_n(x_j^n)-c_n)\to0, \qquad w_j^n\Delta_n'(x_j^n)\to0,
\end{align}
where $\Delta_n$ denotes the witness function associated with the $n$th configuration.
This is a strong asymptotic KKT (s-AKKT) sequence in the sense of \citet{börgens2020newconstraintqualificationsoptimization}; note that the distinction between strong and weak conditions is immaterial in our finite-dimensional setting.
Importantly, these conditions describe asymptotic stationarity of the joint weight and location optimisation problem; the weights need not be fully optimised at each step.
We may assume that at least one particle location escapes to $\pm\infty$ with nonzero limiting weight, since otherwise the finite-position analysis applies.

To study such sequences, we must track not only the particle locations, but also their pairwise distances.
Let $\bar{\mathbb R}=\mathbb R\cup\{-\infty,+\infty\}$ denote the two-point compactification of the real line.
Embed the sequence $\boldsymbol{x}^n\in\mathbb R^N$ into $\bar{\mathbb R}^N\times\bar{\mathbb R}^{N(N-1)/2}$ via
\begin{align*}
    \boldsymbol{x}^n \longmapsto (\boldsymbol{x}^n,\mathbf d^n), \qquad
    d_{ij}^n = |x_i^n-x_j^n|.
\end{align*}
This extends the usual compactification by retaining information about the distances between particles escaping to infinity.
By Tychonoff's theorem, the resulting space is compact, so we may pass to a convergent subsequence.
Since the simplex is also compact, we may assume that $(\mathbf w^n,\boldsymbol{x}^n,\mathbf d^n)\to(\mathbf w^\infty,\boldsymbol{x}^\infty,\mathbf d^\infty)$.
Moreover, $c_n$ is bounded by the simplex constraints and the asymptotic stationarity conditions, so, passing to a further subsequence if necessary, we may assume $c_n\to c$.
The limiting stationarity conditions are then given by
\begin{align}
\label{cAStationary0}
    w_j^\infty(\Delta_\infty(x_j^\infty)-c)=0, \qquad \Delta_\infty(x_j^\infty)\geq c, \qquad w_j^\infty\Delta_\infty'(x_j^\infty)=0, \qquad j=1,\ldots,N,
\end{align}
where evaluations at infinite locations are understood as limits along the original sequence.
We first show that particles with nonzero limiting weights cannot remain at a finite, positive distance from one another if at least one escapes to infinity.
\begin{proposition}
\label{prop: infsep}
Suppose $w_i^\infty,w_j^\infty$ are nonzero.
If $x_i^\infty\in\{\pm\infty\}$, then $d_{ij}^\infty\in\{0,\infty\}$.
\end{proposition}
\begin{proof}
    The conclusion is immediate if $x_j^\infty$ is finite or if the particles escape to opposite infinities.
    Otherwise, suppose that some particles escape to the same infinity while maintaining finite pairwise distances, with at least two distinct limiting locations.
    Consider an outermost particle with positive limiting weight within such a group.
    In the limit, interactions with particles outside the group vanish, while the particles at distinct locations within the group exert an outward force on the outermost particle.
    Since the weights are nonnegative, these contributions cannot cancel, and hence its limiting derivative cannot vanish.
    This contradicts the stationarity condition \eqref{cAStationary0}.
\end{proof}
We can therefore define a reduced asymptotic stationary configuration by removing zero-weight particles and merging particles for which $d_{ij}^\infty=0$.
Write $\vartheta_1^\infty,\ldots,\vartheta_A^\infty$ for its distinct locations and $\varsigma_1^\infty,\ldots,\varsigma_A^\infty$ for their positive total weights.
By Prop.~\ref{prop: infsep}, distinct locations at infinity are infinitely separated, so the distance coordinates are no longer needed.
For the reduced configuration, the stationarity conditions give
\begin{equation}
    \label{aStationarycred}
    \Delta_\infty(\vartheta_j^\infty)=c, \qquad \Delta_\infty'(\vartheta_j^\infty)=0, \qquad j=1,\ldots,A.
\end{equation}
We first revisit Prop.~\ref{prop: cpos}.
Its original interpolation argument does not extend directly to infinite locations, since kernel sections and their derivatives centred at escaping particles vanish at fixed finite points.

\begin{proposition}
\label{prop: asymptoticposweight}
Suppose we are at a spurious asymptotically stationary configuration of $F(\mathbf w,\boldsymbol{x})$.
Then $c>0$.
Moreover, if $\vartheta_i^\infty\in\{\pm\infty\}$, then $\varsigma_i^\infty=c$.
\end{proposition}
\begin{proof}
    By assumption, at least one particle location escapes to infinity with positive limiting weight.
    Let $\vartheta_i^\infty\in\{\pm\infty\}$ denote such a location.
    By Prop.~\ref{prop: infsep}, all other distinct locations with nonzero limiting weight are infinitely separated from it.
    Consequently, for the Gaussian kernel,
    \begin{equation*}
        \Delta_\infty(\vartheta_i^\infty)=\varsigma_i^\infty k(\vartheta_i^\infty,\vartheta_i^\infty)=\varsigma_i^\infty.
    \end{equation*}
    The stationarity condition \eqref{aStationarycred} therefore gives $\varsigma_i^\infty=c$.
    Since $\varsigma_i^\infty>0$, it follows that $c>0$.
\end{proof}
\vspace{-2mm}
Applying the argument of Prop.~\ref{prop: simp max principle} to the finite-position particles recovers the corresponding maximum-principle conclusions.
The collision result in Prop.~\ref{prop: collisionpoints} also extends essentially unchanged, while Prop.~\ref{prop: neighbourhoodmiss} holds in a stronger form.

We can also separate the contributions to the MMD from particles at finite and infinite locations.
Suppose there are $I$ escaping locations with nonzero total weight.
By Props.~\ref{prop: infsep} and \ref{prop: asymptoticposweight}, these locations are infinitely separated and each has weight $c$.
Their contribution to the limiting squared MMD is therefore $Ic^2\leq c$, since $Ic\leq1$ by the simplex constraint.
Thus, these contributions are not themselves problematic.

The main obstruction arises when extending Prop.~\ref{prop: Lcbound}, whose proof relies on the weights summing to one.
A corresponding argument in the asymptotic setting introduces an additional term $Ic$, which need not vanish as $c\to0$.
For example, consider $N$ particles with equal weights $1/N$ and locations $x_j^n=jn$.
As $n\to\infty$, this sequence satisfies the asymptotic stationarity conditions with $c=1/N$, while all $I=N$ particles escape, giving $Ic=1$.
Although such a configuration need not arise along the trajectories of a particular particle flow, it demonstrates that limiting stationarity alone is insufficient to control particle escape.

\section{Numerical illustration}
\label{app:numerics}
In this appendix we provide a numerical illustration of the results established in the main text.
We emphasise that these experiments are illustrative, since a finite numerical search ultimately cannot establish the absence of additional local minima.
Nonetheless, the observed landscapes are qualitatively consistent with our theory.
This is particularly true in the well-specified relaxed setting, where the disappearance of positive-error local minima once $N\geq K$ closely mirrors the sharp guarantee of Theorem~\ref{thm: relaxbound}.
The experiments also illustrate the broader tendency for the local minima encountered numerically to improve as the number of particles increases.

\textbf{Well-specified targets.} For the well-specified setting, we sample 30 three-atom targets with locations $(0,U,24)$, where $U\sim\operatorname{Unif}(0,24)$, and independent weights drawn from $\operatorname{Dirichlet}(1,1,1)$. We use this set of 30 throughout. 
This gives targets with varying atom separations and masses.
For each target, we use $N=1,\dots,9$ particles in both formulations.
Since these targets have $K=3$ atoms, Theorem~\ref{thm: relaxbound} guarantees that the relaxed problem has no spurious minima for $N\geq 3$.

\textbf{Compactly supported targets}. For the general compactly supported setting, we additionally consider the continuous target $P^* =\operatorname{Unif}[0,12]$, with $N=1,\ldots,16$ in both formulations.
This provides a numerical illustration of the approximation regime of Theorems~\ref{thm: relaxepsbound} and~\ref{thm: SimpEpsBound}.
All kernel integrals involving the uniform target are evaluated analytically, without discretising $P^*$.

\textbf{Initialisation.} For each problem, we use 96 initialisations, with the same position initialisations used for the simplex and relaxed formulations.
The weights are initialised from $\operatorname{Dirichlet}(3,\ldots,3)$ to reduce the frequency of boundary starts.
For the positions, let $l$, $m$, and $r$ denote left, middle, and right reference locations. 
For the atomic targets we take $(l,m,r)=(0,U,24)$, corresponding to the three target atoms, while for the uniform target we take $(l,m,r)=(0,6,12)$, corresponding to the two support endpoints and their midpoint.
We then use 16 starts from each of six position schemes: independent uniform draws on $[l-1,r+1]$, $[l-0.5,m+0.5]$, or $[m-0.5,r+0.5]$; equally spaced grids on $[l,m]$ or $[m,r]$ with independent $\mathcal N(0,0.25^2)$ perturbations; and independent uniform draws from ${l,m,r}$ with independent $\mathcal N(0,0.5^2)$ perturbations.
For the atomic targets, the two subinterval schemes are included specifically to search for local minima which fail to represent one of the target atoms, as in the geometric separation underlying Proposition~\ref{prop: neighbourhoodmiss}.

\textbf{Optimisation and certification.} 
We optimise using Sequential Least Squares Programming (SLSQP) with analytic derivatives, tolerance $10^{-13}$, and at most 1400 iterations.
For the atomic targets, the candidate solutions returned by SLSQP are subsequently refined using Newton's method at 100 decimal digits.
For the uniform target and each $N\geq2$, we first screen the SLSQP solutions using preliminary gradient and Hessian checks.
We then refine the remaining candidate with the largest MMD.
We retain positive-error points only when the gradient infinity norm is below $10^{-65}$, the relative Newton correction is below $10^{-45}$, and the joint Hessian is positive definite, restricted to mass-preserving variations in the simplex case.
For the simplex problem, we additionally require all weights to be strictly positive.
Every retained point is checked again at 130 decimal digits.
All computations are carried out in Python 3.13.1 \citep{python}, NumPy \citep{numpy}, SciPy \citep{scipy}, gmpy2 \citep{gmpy2} and mpmath \citep{mpmath} on an arm64 macOS CPU. 
Plots were produced with Matplotlib \citep{matplotlib}. NumPy, SciPy and mpmath use BSD-3-Clause licences; gmpy2 uses an LGPL-3.0-or-later licence; Matplotlib uses its PSF-based licence. Python uses the Python Software Foundation License Version 2.

\textbf{Results.}
Our numerical results are qualitatively consistent with our theory. 
This is particularly clear in the well-specified relaxed problem, where positive-error local minima are found for $N<3$, whilst for every $N\geq K=3$ the numerical search retains only exact, or numerically indistinguishable from exact, minima.
This is consistent with Theorem~\ref{thm: relaxbound}, which gives the sharp threshold $N\geq K=3$ in this setting.
In the simplex case, positive-error minima persist for some targets throughout the range of particle numbers considered.
However, as $N$ increases, no positive-error minima are found for an increasing proportion of targets.
This is consistent with Theorem~\ref{thm: SimpBound}, which gives a larger, target-dependent sufficient threshold than in the relaxed case.
For the continuous uniform target (Figure~\ref{fig}b), the MMD of the retained minima decreases steadily with $N$ in both formulations.
This is qualitatively consistent with Theorems~\ref{thm: relaxepsbound} and~\ref{thm: SimpEpsBound}, which show that sufficiently many particles rule out local minima whose MMD remains bounded away from zero.
\begin{figure}[htbp]
\centering
\includegraphics[width=\textwidth, trim={0 4mm 0 2mm}, clip]{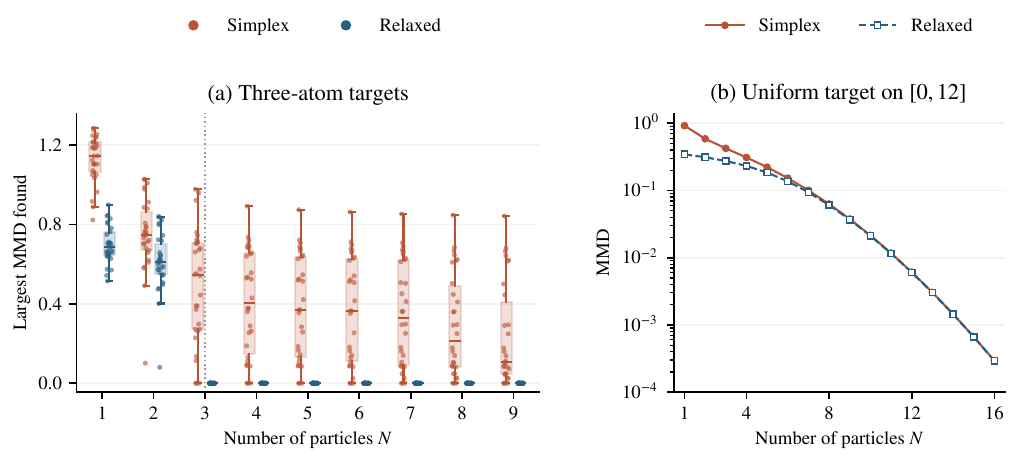}
\vspace{-6mm}
\caption{A numerical illustration of the finite-particle landscapes.
(a) Largest MMD among the retained local minima, for each of the 30 well-specified three-atom targets.
Dots correspond to individual targets, while boxes show medians and quartiles with $1.5$-IQR whiskers.
For $N\geq3$, the target can be represented exactly; hence a plotted value of zero means that the numerical search retained no positive-error local minimum.
(b) MMD of the retained local minima for the continuous target $P^*=\operatorname{Unif}[0,12]$, shown on a logarithmic scale.
The $N=1$ solutions are obtained analytically.}
\label{fig}
\end{figure}

\end{document}